\def\year{2020}\relax
\documentclass[letterpaper]{article} 
\usepackage{aaai20}  
\usepackage{times}  
\usepackage{helvet} 
\usepackage{courier} 
\usepackage[hyphens]{url} 
\usepackage{graphicx}
\usepackage{graphicx} 
\usepackage{dirtytalk}
\usepackage{amsmath}
\usepackage{amsfonts}
\usepackage{amsthm}
\usepackage{bm}

\usepackage{algorithm}
\usepackage[noend]{algpseudocode}

\newcounter{alphasect}
\def\alphainsection{0}

\let\oldsection=\section
\def\section{
  \ifnum\alphainsection=1
    \addtocounter{alphasect}{1}
  \fi
\oldsection}

\renewcommand\thesection{\ifnum\alphainsection=1\Alph{alphasect}\else\arabic{section}\fi}

\newenvironment{alphasection}{
  \ifnum\alphainsection=1
    \errhelp={Let other blocks end at the beginning of the next block.}
    \errmessage{Nested Alpha section not allowed}
  \fi
  \setcounter{alphasect}{0}
  \def\alphainsection{1}
}{
  \setcounter{alphasect}{0}
  \def\alphainsection{0}
}

\newtheorem{theorem}{Theorem}[section]

\newtheorem{lemma}[theorem]{Lemma}

\usepackage{booktabs}

\title{Policy Search by Target Distribution Learning for Continuous Control}
\author{
  Chuheng Zhang \\
  IIIS, Tsinghua University\\
  {zhangchuheng123@live.com} \\
  \And
  Yuanqi Li \\
  IIIS, Tsinghua University\\
  {timezerolyq@gmail.com} \\
  \And
  Jian Li \\
  IIIS, Tsinghua University\\
  {lapordge@gmail.com} \\
}

\begin{document}

\maketitle

\begin{abstract}
    It is known that existing policy gradient methods (such as vanilla policy gradient, PPO, A2C) may suffer from overly large gradients 
    when the current policy is close to deterministic, 
    leading to an unstable training process. 
    We show that such instability can happen even
    in a very simple environment.
	To address this issue, we propose a new method,	called \emph{target distribution learning} (TDL),
	for policy improvement in reinforcement learning. 
	TDL alternates between proposing a target distribution and training the policy network to approach the target distribution. TDL is more effective in constraining the KL divergence between updated policies, and hence 
	leads to more stable policy improvements over iterations. 
	Our experiments show that TDL algorithms perform comparably to (or better than) state-of-the-art algorithms 
	for most continuous control tasks in the MuJoCo environment
	while being more stable in training. 
\end{abstract}

\section{Introduction}

Reinforcement learning (RL) algorithms can be broadly divided into value-based methods and policy search methods. When applied to continuous control tasks, value-based methods, such as \cite{mnih2015human,schaul2015prioritized,wang2015dueling,van2016deep,dabney2018implicit},
need additional treatments
to convert the learned value function to executable policies \cite{gu2016continuous,novati2018remember}. 
On the other hand, policy search methods directly improve a policy for continuous control. Among others, policy gradient-based methods have been
shown to be quite effective in searching good policies, e.g., \cite{williams1992simple,sutton1998introduction,silver2014deterministic,lillicrap2015continuous,mnih2016asynchronous}.
These methods first compute the gradient of the performance measure with respect to the parameters of the policy network and then update the parameters via stochastic gradient ascent.
In order to ensure that the policy improves in terms of the performance measure over the iterations, the policy improvement theorem \cite{kakade2002approximately} suggests that the policy update should not be too large over one iteration.
Specifically, policy improvement requires a regularization on the \emph{state-action space} to avoid destructively large updates, i.e., the probability distributions over the action space of the old and new policies conditioned on a state should not vary too much.

Influenced by the policy improvement theorem,
several effective algorithms have been proposed in the literature. 
TRPO \cite{schulman2015trust} and ACKTR \cite{wu2017scalable} both update the policy subject to a constraint in the state-action space (trust region).
ACER \cite{wang2016sample} adopts a trust region optimization method that clips the policy gradient in the state-action space to constrain the policy update. PPO \cite{schulman2017proximal} designs a clipped surrogate objective that approximates the regularization.
PPO has been proven to be quite effective and is relatively simple-to-implement, thus becomes a quite popular method. 
In a very simple environment (see Section \ref{sec:instability}), we observe some weakness of PPO in our experiment:
when the policy is near-deterministic, the gradient may explode which leads to instability.
Moreover, PPO performs multiple epochs of minibatch updates on the same samples to fully utilize the samples. We observe significant performance degradation when increasing the sample reuse (see Section \ref{chap:exp1}). 

We propose a new policy search method, called \emph{target distribution learning (TDL)}, that improves the policy over iterations. In each iteration, the \emph{action distribution} conditioned on each encountered state is produced by the policy network. The actions are drawn from these distributions and used to interact with the environment. Then, TDL proposes better action distributions (called \emph{target distributions}) to optimize the performance measure and updates the policy network to approach the target distributions. 

The contributions of our work are summarized as follows: 
\begin{enumerate}
    \item 
    (Section \ref{sec:instability})
    We show experimentally that PPO (even combined with some commonly used remedies) suffers from an instability issue during the training even in a very simple environment.
    \item
    (Section \ref{sec:TDL}) We propose a new policy improvement method TDL that
    retains the simplicity of PPO
    while avoids the instability issue. 
    We also propose three algorithms based on TDL, all of which set target distributions within a trust region and thus ensure the policy improvement. 
    We provide theoretical guarantee that the target distribution 
    is close to the old distribution in 
    terms of KL divergence (Appendix \ref{app:KL_bound}).
    Two algorithms set the target distributions following an update rule of evolutionary strategy (ES) \cite{rechenberg1973evolutionsstrategie}. 
    Unlike previous work (such as \cite{salimans2017evolution,mania2018simple,liu2019trust}) which used ES to search over the parameter space directly, we incorporate the idea in ES to propose better action distributions.
    Moreover, 
    we show that the target distributions proposed by one of our algorithms based on ES indicate a desirable direction
    and illustrate that the algorithm can better
    prevent premature convergence (Appendix \ref{app:ES_theory}). 
    \item (Section \ref{sec:exp}) We conduct several 
    experiments to show that our algorithms perform comparably to (or better than) several state-of-the-art algorithms on benchmark tasks. Moreover, we show that our algorithms are more effective to realize a regularization in the state-action space than TRPO and PPO, and can increase the sample reuse without significant performance degradation.
\end{enumerate}

\section{Preliminaries}
\label{sec:prelim}

A Markov Decision Process (MDP) for continuous control is a tuple $(\mathcal{S}, \mathcal{A}, P, R, \gamma)$ specifying the state space $\mathcal{S}$, 
the continuous action space $\mathcal{A} \subseteq \mathbb{R}^d$, 
the state transition probability $P(s_{t+1}|s_t, a_t)$, 
the reward $R(r_t|s_t,a_t)$ and the discount factor $\gamma$. 
Let $\pi$ denote a stochastic policy $\pi: \mathcal{S}\times\mathcal{A} \to [0, 1]$. 
In this paper, it is specified by a probability distribution whose statistical parameter is given by a neural network with parameter $\theta$, 
i.e., $\pi(a_t|\phi_\theta(s_t))$ where $\phi_\theta(s_t)$ denotes the statistical parameter (e.g., the mean and standard deviation of a Gaussian). 
We call this probability distribution \emph{action distribution}. 
The value function is defined as $V^\pi(s) := \mathbb{E} [\sum_{t=0}^\infty \sum_{s'} p(s_t = s'| s_0 = s, \pi_\theta) \gamma^t r_t]$ 
for each $s\in\mathcal{S}$. 
The corresponding Q-function is defined as $Q^\pi(s, a) := \mathbb{E} [\sum_{t=0}^\infty \sum_{s'} p(s_t = s'| s_0 = s, a_0 = a, \pi_\theta) \gamma^t r_t] $ 
for each $s\in\mathcal{S}$ and $a\in\mathcal{A}$. 
The advantage function for each action $a$ in state $s$ is defined as $A^\pi(s,a) = Q^\pi(s,a) - V^\pi(s)$. 
The goal is to maximize the expected cumulative reward from an initial state distribution, 
i.e., $\max_\pi \eta(\pi) := \mathbb{E}_{s_0} [V^\pi (s_0)]$.

In this paper, we use multivariate Gaussian distributions with diagonal covariance matrices as the action distributions 
for the stochastic policy. 
In this case, the statistical parameter $\phi_\theta(s)$ has two components, the action mean $\mu_\theta(s) \in \mathbb{R}^d$ and the diagonal elements of covariance matrix (variance) $\sigma^2_\theta(s) \in \mathbb{R}^d$. In each iteration, the new policy $\pi(a|s) = \mathcal{N}(a|\mu_\theta(s), \sigma_\theta(s))$ is updated from an old policy $\pi^{\text{old}}(a|s) = \mathcal{N}(a|\mu^{\text{old}}(s), \sigma^{\text{old}}(s))$.

In each iteration, the policy network is updated to maximize the following surrogate objective subject to a Kullback--Leibler (KL) divergence \cite{kullback1951information} constraint to prevent destructively large updates. The formulation is first proposed by \cite{schulman2015trust} based on \cite{kakade2002approximately}.
$$
L(\theta) = \mathbb{E}_{s \sim \rho_{\pi^{\text{old}}}} \left[ \sum_a \mathcal{N}(a|\mu_\theta(s), \sigma_\theta(s)) A^{\pi^{\text{old}}} (s, a) \right]
$$
\begin{equation}
    s.t. \, \max_{s\in\mathcal{S}} KL\left(\mathcal{N}(\mu^{\text{old}}(s), \sigma^{\text{old}}(s))||\mathcal{N}(\mu_\theta(s), \sigma_\theta(s))\right) \le \delta ,
    \label{eq:policy_improve}
\end{equation}
where $\rho_\pi(s) = \mathbb{E}_{s_0} [ \sum_{t=0}^\infty \gamma^t p(s_t = s | s_0, \pi) ]$ is the state visitation frequency and $\text{KL}(\cdot||\cdot)$ indicates the KL divergence between two probability distributions. When $\delta$ is small, a solution of the above 
optimization problem can guarantee 
a policy improvement over the iteration, 
i.e., $\eta(\pi) > \eta(\pi^{\text{old}})$ \cite{schulman2015trust}.

The above optimization objective can be approximated using Monte Carlo samples as follows:
\begin{equation}
\label{eq:mc_loss}
    \hat{L}(\theta) = \dfrac{1}{T} \sum_{t=1}^T \left[ \hat{A}_t \dfrac{\mathcal{N}(a_t|\mu_\theta(s_t), \sigma_\theta(s_t))}{\mathcal{N}(a_t|\mu^{\text{old}}(s_t), \sigma^{\text{old}}(s_t))} \right] ,
\end{equation}
where $s_t$ and $a_t$ are the samples of the state and the action, respectively, at timestep $t$ following $\pi^{\text{old}}$.
$\hat{A}_t:=\hat{A}^{\pi^{\text{old}}}(s_t, a_t)$ is an estimator of the advantage function at timestep $t$. One popular choice is to use generalized advantage estimator (GAE) \cite{schulman2015high} as the advantage estimator. We note that TRPO and PPO are based on the above formulation.

\section{Instability Issue of Previous Methods}
\label{sec:instability}

In this section, we show that the gradient of the objective $\hat{L}(\theta)$ (in \eqref{eq:mc_loss}) with respect to $\theta$ may explode when the policy is near-deterministic, i.e., $\sigma_\theta(\cdot)$ is small, which may lead to instability in training. 

Let us consider a case where the standard deviation of the action distribution $\sigma$ is state independent and thus itself is a parameter of the policy network. Define $\hat{L}_t(\theta) = \hat{A}_t \dfrac{\mathcal{N}(a_t|\mu_\theta(s_t), \sigma)}{\mathcal{N}(a_t|\mu^{\text{old}}(s_t), \sigma^{\text{old}})}$. By the standard chain rule, one can
see that the gradient with respect to $\theta$ is as follows: 
\begin{equation}
\dfrac{\partial \hat{L}(\theta)}{\partial \theta} = \dfrac{1}{T} \sum_{t=1}^T \left[\hat{L}_t(\theta) \dfrac{\partial \log \mathcal{N}(a_t|\mu_\theta(s_t), \sigma)}{\partial \mu_\theta(s_t)} \dfrac{\partial \mu_\theta(s_t)}{\partial \theta} \right].
\end{equation}
Moreover, the gradient of the logarithm of the probability density with respect to the mean is
\begin{equation}
\dfrac{\partial \log \mathcal{N}(a_t|\mu_\theta(s_t), \sigma)}{\partial \mu_\theta(s_t)} = \dfrac{a_t - \mu_\theta(s_t)}{\sigma^2}.
\end{equation}
Therefore, the gradient with respect to $\theta$ is inversely proportional to $\sigma$ for a typical sample $a_t$.
When the policy is near-deterministic, i.e., $\sigma$ is small, the gradient with respect to $\theta$ becomes large. So, it is likely that, given a state $s$, the mean of the action distribution conditioned on this state $\mu_\theta(s)$ is updated to a place far away from the previous mean $\mu^{\text{old}}(s)$,
which may already be close to optimal. 
This thus leads to a \say{bad} action in the next iteration. Notice that other policy gradient-based algorithms involving the gradient of a probability density function (such as vanilla policy gradient, A2C) may suffer from the same issue. \footnote{TRPO does not suffer from such issue since it performs a line search for the step size.} \cite{zhao2011analysis} has similar results that the variance of the policy gradient update is inversely proportional to the square of the standard deviation in two policy gradient algorithms, REINFORCE \cite{williams1992simple} and PGPE \cite{gruttner2010multi}.

Now, we experimentally show that PPO suffers from such instability issue
even in a simple environment as follows: In each round, the environment samples a state $s\sim U([0,1]^1)$. The agent receives the state, performs an action $a\in\mathbb{R}^1$ and suffers a cost (negative reward) $c(a) = a^2$. The objective is to minimize the one-step cost, i.e., $\min_\theta \mathbb{E}[\sum_a \pi_\theta(a|s) c(a)]$. Notice that the cost is independent of the state but the state is still fed as an input to the policy network. It is obvious that the optimal policy should play $a=0$ with probability 1 for any state, which is a deterministic policy. Our experiment shows that PPO suffers from the aforementioned instability issue, resulting in an oscillating and diverging behavior. 
On the other hand,  our new method TDL (see Section \ref{sec:TDL}) circumvents the computation of the gradient of a probability density function, hence does not suffer from such instability issue.
In practice, there are some common tricks that attempt to solve the instability issue,
such as adding an entropy term in the loss function, setting a minimum variance threshold below which the variance of the action distribution is not allowed to decrease further, or using a small clip constant in PPO. 
We also adopt these tricks for PPO and compare them with our method. The entropy term can stabilize the training 
but leads to a worse asymptotic performance. The minimum variance
constraint can achieve a smaller cost but is still unstable. A small clip constant may delay but does not prevent the instability. 
We show the result in Figure \ref{fig:convex}.

\begin{figure}[htbp]
   \centering
   \includegraphics[width=0.95\columnwidth]{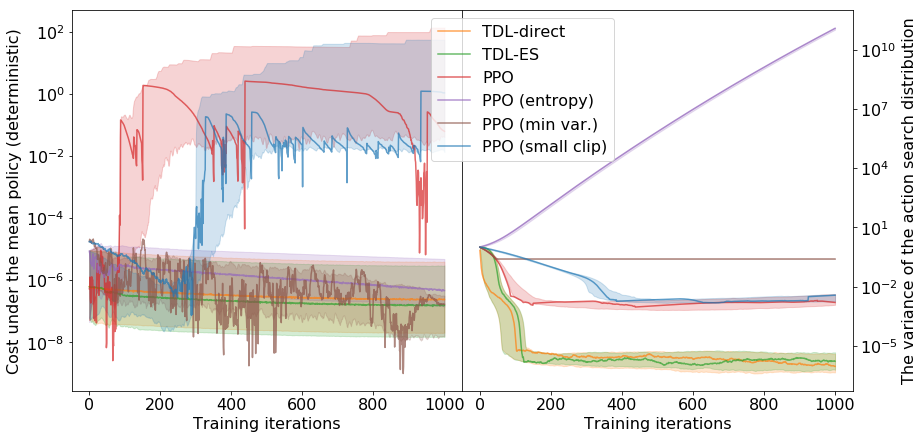}    
   \caption{Median performance out of 100 independent runs for each algorithm in the simple environment described in Section \ref{sec:instability}. The shaded areas indicate the 10\% and 90\% quantiles. Left. The cost from executing the mean of the action distribution along the training. Right. The variance of the action distribution along the training.}
   \label{fig:convex}
\end{figure}

\section{Target Distribution Learning}
\label{sec:TDL}

Instead of optimizing the objective $L(\theta)$ with respect $\theta$ directly, TDL solves the constrained optimization problem in two steps: TDL first proposes statistical parameters that specify action distributions (targets) under which the expected advantage function of the old policy is improved. Then, TDL trains the policy network to match the targets.

In the first step, for each state sample $s_t$, TDL proposes a target distribution whose statistical parameters 
attempt to maximize $L_{t,1}(\mu, \sigma)$, the surrogate objective on state $s_t$ (cf. equation \eqref{eq:policy_improve}), where
\begin{equation}
L_{t,1}(\mu, \sigma) = \mathbb{E}_{a\sim \mathcal{N}(\mu, \sigma)}\left[ A^{\pi^{\text{old}}}(s_t, a) \right]
\label{eq:opt1}
\end{equation}
or maximize $L_{t,2}(\mu, \sigma)$, the probability that the target policy is improved over the old value function, where 
\begin{equation}
L_{t,2}(\mu, \sigma) = \mathbb{E}_{a\sim \mathcal{N}(\mu, \sigma)}\left[ \mathbb{I}\{ A^{\pi^{\text{old}}}(s_t, a) > 0 \} \right]
\label{eq:opt2}
\end{equation}
while being subject to the following constraint:
\begin{equation}
KL(\mathcal{N}(\cdot|\mu^{\text{old}}(s_t), \sigma^{\text{old}}(s_t))||\mathcal{N}(\cdot|\mu, \sigma)) \le \delta.
\label{eq:opt_constr}
\end{equation}

In the second step, the policy network learns to match the proposed targets by minimizing the mean squared error with respect to these target statistical parameters.

Notice that typically only one estimate $\hat{A}_t := \hat{A}^{\pi^{\text{old}}}(s_t, a_t)$ is known. Therefore, the above optimization problems cannot be solved exactly and there is a tradeoff between exploitation and exploration for the target distribution, i.e., we can move the target distribution to the \say{best} area indicated by the estimate and shrink the variance of the distribution (exploit) or we can increase the variance for a better estimation of the advantage function (explore). In policy gradient methods, the mean and the variance of the action distribution are updated jointly subjected to the law given by the gradient of the probability density function. However, in TDL, the mean and the variance can be updated independently to implement different action search strategies. 

Next, we propose three algorithms based on TDL, \emph{TDL-direct}, \emph{TDL-ES} and \emph{TDL-ESr}. The three algorithms differ in the way they propose target distributions. 
The pseudocode for TDL is shown in Algorithm \ref{algorithm}, the details of which are described in the following paragraphs.

\begin{algorithm*}
\caption{Target learning}
\label{algorithm}
\begin{algorithmic}[1]
\State Number of timesteps in one iteration $T$, minibatch size $M$, number of epochs $E$
\State Initialize the action distribution of the policy $\mathcal{N}(\mu_\theta(\cdot), \sigma_\theta(\cdot) = \sigma^{1 / (\varphi + 1)} \tilde{\sigma}_\theta(\cdot)^{\varphi / (\varphi + 1)})$
\State Initialize the critic network $V_\phi(\cdot)$
\For {$i=0,1,2,\cdots$} 
	\State Interact with the environment and obtain $T$ on-policy transitions $\{(s_t, a_t, r_t, s_{t+1})\}$
	\State Calculate the Monte Carlo return for each transition $R_t = \sum_{t'=t}^T \gamma^{t'-t} r_{t'}$
	\State Calculate the advantage function estimate $\hat{A}_t$ for each transition (by GAE)
	\State Calculate the target standard deviation $\hat{\sigma}_t$ following \eqref{eq:state_dependent_variance} 
	\State \textbf{if} TDL-direct \textbf{then} calculate the target means $\hat{\mu}_t$ following (\ref{eq:mean_direct}) 
	\State \textbf{if} TDL-ESr \textbf{then} revise the sampled actions following \eqref{eq:revise}
	\State \textbf{if} TDL-ES \textbf{or} TDL-ESr \textbf{then} calculate the target means $\hat{\mu}_t$ following (\ref{eq:mean_ES})
	\For {$j=1:ET/M$}
		\State Sample a minibatch that contains $M$ transitions
		\State Update the policy network to minimize $\frac{1}{M}\sum_{t=1}^M(\hat{\mu}_t - \mu_\theta(s_t))^2$ and $\frac{1}{M}\sum_{t=1}^M(\hat{\sigma}_t - \tilde{\sigma}_\theta(s_t))^2$ on the minibatch
		\State Update the critic network to minimize $\frac{1}{M}\sum_{t=1}^M(R_t - V_\phi(s_t))^2$ on the minibatch
	\EndFor
	\State Update $\sigma$ to $\hat{\sigma}$ defined in \eqref{eq:state_independent_variance} and $\sigma_\theta(\cdot)$ following \eqref{eq:variance_formula}
\EndFor
\end{algorithmic}
\end{algorithm*}

\subsection{Target variance}

The three algorithms update the variance in the same way described as follows. The update rule allows an adaptation for exploration while changes the variance slowly which prevents premature convergence and violation of the constraint.

Inspired by the self-adaption technique \cite{hansen2001completely,hansen2000invariance}, given one state sample, action sample and the corresponding advantage estimate $(s_t, a_t, \hat{A}_t)$, the target variance is proposed as:
\begin{equation}
\label{eq:state_dependent_variance}
\hat{\sigma}_t^2 =
(a_t - \mu^{\text{old}}(s_t))^2 \mathbb{I}\{\hat{A}_t > 0\} + (\sigma^{\text{old}}(s_t))^2 \mathbb{I}\{\hat{A}_t \le 0\},
\end{equation}
where $\mu^\text{old}(\cdot)$ and $\sigma^\text{old}(\cdot)$ are the values of $\mu_\theta(\cdot)$ and $\sigma_\theta(\cdot)$ in the last iteration.
When the advantage estimate is positive, if the action sample is within the one-sigma range, the target variance will be smaller than the old variance, otherwise it gets larger than the old variance for further exploration. When the advantage estimate is negative, the target variance remains the same as the old variance. Notice that when the advantage estimator is non-informative, the target variance remains the same as the old in expectation, i.e., $\mathbb{E}_{a_t \sim \mathcal{N}(\mu^{\text{old}}(s_t), \sigma^{\text{old}}(s_t))} [\hat{\sigma}_t^2] = \sigma^{\text{old}}(s_t)^2$. This prevents a drift of the action distribution when the critic is not well learned.

To prevent large variance update that may lead to a violation of the KL divergence constraint, the variance on a state $s_t$ is designed to be the average of a state independent component and a state dependent component as follows:
\begin{equation}
\label{eq:variance_formula}
\sigma_\theta(s_t) = \sigma^{1 / (\varphi + 1)} \tilde{\sigma}_\theta(s_t)^{\varphi / (\varphi + 1)}
\end{equation}
where $\varphi > 0$ is a hyperparameter that controls the degree to which the variances on different states are updated independently. $\tilde{\sigma}_\theta(s_t)$ is the state dependent component which is a neural network trained to minimized the MSE w.r.t. $\hat{\sigma}_t$. $\sigma$ is the state independent component and directly updated to $\hat{\sigma}$ in each iteration, where
\begin{equation}
\label{eq:state_independent_variance}
\hat{\sigma}^2 = \dfrac{1}{T}\sum_{t=1}^T \hat{\sigma}_t^2.
\end{equation}

The state independent component is based on all the $T$ samples in the iteration and therefore allows a global adaptation for exploration while changes slowly. With the state independent component, the variance changes slowly in one iteration. Empirically in our later experiments on Mujoco \cite{todorov2012mujoco}, with $\varphi =1$ and $T=2048$, each dimension of $\sigma_\theta(s_t) / \sigma^{\text{old}}(s_t)$ for any state falls within $[1-\epsilon, 1+\epsilon]$ for a $\epsilon \le 0.01$. 

\subsection{TDL-direct algorithm}

Consider a state sample $s_t$, an action sample $a_t$ and its advantage estimate $\hat{A}_t$. Given that the update of the variance is small in each iteration, TDL-direct sets the target mean $\hat{\mu}_t$ to $\mu$ that maximizes $\mathcal{N}(a_t\,| \mu, \sigma^\text{old}(s_t)) \hat{A}_t$ (i.e., the Monte Carlo estimate of $L_{t,1}$) subject to the constraint in \eqref{eq:opt_constr}. 

Recall that the action $a_t$ is sampled from $\mathcal{N}(\mu^{\text{old}}(s_t),  \sigma^{\text{old}}(s_t))$. Hence,
we can write $a_t = \mu^{\text{old}}(s_t) + y_t \sigma^{\text{old}}(s_t)$,
where $y_t\sim \mathcal{N}(0,I)$.
Let $\alpha > 0$ be a hyperparameter controlling the size of the trust region. The target mean for the sample at timestep $t$ is proposed as
\begin{equation}
\label{eq:mean_direct}
\hat{\mu}_t = \mu^{\text{old}}(s_t) + 
\text{sign} (\hat{A}_t) \min \left(1, \dfrac{\sqrt{2\alpha}}{\|y_t\|_2}\right) 
y_t \sigma^{\text{old}}(s_t), 
\end{equation}
where $\text{sign} (\cdot)$ is the sign function.
When $\hat{A}_t > 0$, $\mathcal{N}(a_t\,| \mu, \sigma) \hat{A}_t$ can be maximized by setting $\mu = a_t$. When $\hat{A}_t \le 0$, it is preferred that $\mu$ should be as far away from $a_t$ as possible. However, this may violate the constraint in (\ref{eq:opt_constr}). Thus, we clip the amount of the change from $\mu^{\text{old}}(s_t)$ to $\hat{\mu}_t$ such that $KL(\mathcal{N}(\cdot|\mu^{\text{old}}(s_t), \sigma^{\text{old}}(s_t))||\mathcal{N}(\cdot|\hat{\mu}_t, \sigma(s_t))) \le d \alpha (1+2\epsilon) + o(\epsilon^2) \approx d \alpha$, 
recalling that the standard deviation changes within $[1-\epsilon, 1+\epsilon]$ in each iteration.
In Appendix \ref{app:KL_bound}, we show that the above clip operation
can guarantee that the KL constraint is satisfied (by leveraging 
the fact that the KL divergence between two Gaussian distributions has a closed-form formula).

\subsection{TDL-ES algorithm}

$(1+1)$-ES, one of the evolutionary strategies (ES) \cite{beyer2002evolution}, can be used to maximize $L_{t,1}$. 
It is a family of optimization techniques for finding a distribution $\mathcal{D}$ that maximizes $\mathbb{E}_{x\sim\mathcal{D}} [f(x)]$, for a blackbox \emph{fitness} function $f$.
Natural evolutionary strategy (NES) \cite{wierstra2014natural} provides an algorithm based on $(1+1)$-ES that iteratively updates a Gaussian distribution to optimize the objective along a natural gradient direction. Specifically, in each iteration, an \emph{offspring} $x$ is sampled from the Gaussian distribution $\mathcal{N}(\mu^{\text{old}}, \sigma^{\text{old}})$ centering at the \emph{parent} $\mu^{\text{old}}$ and the distribution is updated based on the comparison between the fitness of the offspring $f(x)$ and that of the parent $f(\mu^{\text{old}})$. 

We observe that the objective $L_{t,1}$
is essentially the same as the objective for 
$(1+1)$-ES. By letting $\mu^{\text{old}} = \mu^{\text{old}}(s_t)$ and $x = a_t$, $\hat{A}_t$, which is an estimate of $Q^{\pi^{\text{old}}}(s_t, a_t) - V^{\pi^{\text{old}}}(s_t)$, can be used to indicate $f(x) - f(\mu^{\text{old}})$. 
In this way, the optimization problem defined in \eqref{eq:opt1} can be solved by $(1+1)$-ES. Therefore, the target mean can be proposed by the update rule for NES, as follows:
\begin{equation}
\label{eq:mean_ES}
\hat{\mu}_t = \mu^{\text{old}}(s_t) + \nu \mathbb{I}\{ \hat{A}_t > 0 \} (a_t - \mu^{\text{old}}(s_t)),
\end{equation}
where $\nu \in (0, 1]$ is the step size. For the 
update of the variance, we still use the aforementioned update rule. 
\footnote{
We found that the variance easily explodes following the NES rules to set the target variance in RL context. 
Hence, we choose to keep using the update rules described in \eqref{eq:state_dependent_variance} and \eqref{eq:state_independent_variance},
where \eqref{eq:state_dependent_variance} can be regarded as the first order Taylor approximation of the NES rules for the variance update.
}

TDL-ES algorithm with target statistical parameters defined in \eqref{eq:state_dependent_variance}, \eqref{eq:state_independent_variance} and \eqref{eq:mean_ES} has the following properties. 

First, for a typical action sample $a_t$, the proposed target statistical parameters satisfy the constraint in \eqref{eq:opt_constr}, i.e., $\mathbb{E} [ KL(\mathcal{N}(\cdot|\mu^{\text{old}}(s_t), \sigma^{\text{old}}(s_t))||\mathcal{N}(\cdot|\hat{\mu}_t, \sigma(s_t))) ] \le \frac{1}{2} d \nu^2 (1+2\epsilon) + o(\epsilon^2) \approx \frac{1}{2} d \nu^2$ (cf. Appendix \ref{app:KL_bound}). 

Second, \eqref{eq:mean_ES} and \eqref{eq:state_dependent_variance} can be regarded as a stochastic gradient ascent step w.r.t. $L_{t, 2}$ for some step sizes $\lambda_\mu$ and $\lambda_\sigma$ (cf. Appendix \ref{app:ES_theory}), i.e., 

\begin{equation}
\label{eq:SGA2}
\hat{\mu}_t = \mu^{\text{old}}(s_t) + \lambda_\mu \left. \dfrac{\partial L_{t,2}(\mu, \sigma^{\text{old}}(s_t))}{\partial \mu}\right\vert_{\mu=\mu^\text{old}(s_t)},
\end{equation}

\begin{equation}
\label{eq:SGA1}
\hat{\sigma}_t^2 = (\sigma^{\text{old}}(s_t))^2 + \lambda_\sigma \left. \dfrac{\partial L_{t,2}(\mu^{\text{old}}(s_t), \sigma)}{\partial \sigma^2}\right\vert_{\sigma=\sigma^\text{old}(s_t)}.
\end{equation}

Third, consider one policy improvement step in TDL-ES and denote $D:=\{ a | Q^{\pi^{\text{old}}}(s_t, a) > V^{\pi^{\text{old}}}(s_t) \}$ for a state $s_t$ which represents the \say{good} areas in the action space indicated by the value functions of the old policy. TDL-ES updates the standard deviation of the action distribution towards the (truncated) \say{radius} of $D$ and the mean of the action distribution towards the \say{center} of $D$. This is appealing, since when the actor performs poorly (leading to a small $V(s_t)$ and a large $D$), it keeps exploring. In addition, when the critic estimate is overly large or small (leading to very large or very small $D$),
the action distribution remains the same in expectation. In contrast, a vanilla policy gradient method under a similar setting updates the variance of the action distribution towards zero and the mean of the action distribution towards $arg\max_a Q^{\pi^{\text{old}}}(s_t, a)$. This may lead to premature convergence. 
See the detail in Appendix \ref{app:ES_theory}.

\subsection{TDL-ESr algorithm}

Both TDL-direct and TDL-ES propose the target mean based solely on the action sample $a_t$ from the state $s_t$ and ignore the temporal structure of MDP. According to the observation that the state representation does not change too fast in adjacent steps, we can revise the formulation for the target mean in TDL-ES (i.e., \eqref{eq:mean_ES}) by the information of $2N+1$ adjacent samples $(a_{t+t'}, \hat{A}_{t+t'}), t' \in [-N, N]$, resulting in a revised version of TDL-ES which we call TDL-ESr. The revised formula is the same as \eqref{eq:mean_ES}, except that we substitute $a_t$ with $\tilde{a}_t$. 
Suppose $a_t = \mu^{\text{old}}(s_t) + y_t \sigma^{\text{old}}(s_t)$ is obtained by sampling $y_t \sim \mathcal{N}(0, I)$. For a revising ratio $r\in [0, 1]$, $\tilde{a}_t$ can be defined as $\tilde{a}_t = \mu^{\text{old}}(s_t) + \tilde{y}_t \sigma^{\text{old}}(s_t)$, where
\begin{equation}
\begin{aligned}
\tilde{y}_t & = (1-r) y_t + r y'_t, \\
y'_t & = \dfrac{ \sum_{t'=-N}^N y_t \max(0, \hat{A}_{t+t'}) }{ \sum_{t'=-N}^N \max(0, \hat{A}_{t+t'}) },
\label{eq:revise}
\end{aligned}{}
\end{equation}
Recall that, in TDL-ES, the mean of the action distribution moves to the direction indicated by $y_t$. This revision makes the mean update tilt to a direction $y'_t$ indicated by adjacent \say{good} samples, i.e., samples with $\hat{A}_{t+t'} > 0$. Consider a case where an action sample $a_t$ yields a large reward and results in a large $\hat{A}_t$, indicating that this is potentially a good direction. In TDL-ESr, the mean updates on the adjacent states will tilt towards this direction. 
This yields a directional exploration.

\section{Related Work}
\label{sec:related_work}

\textbf{Conservative policy iteration.} TDL follows conservative policy iteration (CPI) \cite{kakade2002approximately,scherrer2014approximate,agarwal2019optimality}, which suggests a small policy update in each iteration to ensure monotonic policy improvement and avoid oscillations. Several previous methods \cite{schulman2015trust,schulman2017proximal,wu2017scalable,novati2018remember} also follow CPI and constrain the \emph{parameters} of the policy directly. 
Instead, ACER \cite{wang2016sample} clips the gradient on the \emph{statistical parameters} produced by the policy network. 
Similarly, TDL 
also constrains the policy update in the statistical parameter space 
but proposes target distributions that satisfy the constraint. 
Safe policy iteration \cite{pirotta2013safe} proposes to update the policy with different step sizes on different states to ensure policy improvement and is characterized by faster convergence and guaranteed policy improvement. TDL benefits from the similar idea but with a simpler rule to determine the step sizes on different states.

\noindent\textbf{Setting targets.} Setting target policies is an effective way to guide the learning of the policy network.
In large discrete action domain, AlphaGo Zero \cite{silver2017mastering} uses MCTS to propose a categorical distribution as the target for the policy network to learn. In continuous action domain, MPO \cite{abdolmaleki2018maximum} and Relative Entropy Regularized Policy Iteration \cite{abdolmaleki2018relative} set the target action distribution to be an improvement over the estimated Q-function. TDL uses a state value function which is typically easier to estimate than a Q-function,
and perform in more robust manner upon near-deterministic policies, as shown
in our experiments.

\noindent\textbf{Instability problem.} 
The analysis on the variance of the policy gradient update in 
\cite{zhao2011analysis} implies the instability problem on near-deterministic policies.
In addition, \cite{van2007reinforcement,hamalainen2018ppo} state that updating along the negative direction upon a \say{bad} action sample may cause instability. We empirically show that this may lead to instability but can accelerate the learning process when the action space dimension or the step size is small (cf. Appendix \ref{sec:neg_adv}).

\section{Experiments}
\label{sec:exp}

We conduct several experiments in order
to demonstrate the following:
1) The performance of our algorithms 
on continuous control benchmark tasks is comparable with
(or better than) the state-of-the-art algorithms; 
2) We can safely increase the on-policy sample reuse without damaging the performance of our algorithms;
3) Our algorithms can constrain the maximum KL divergence across the state space more effectively than TRPO and PPO.

\begin{figure*}[htb]
   \centering
   \includegraphics[width=0.8\textwidth]{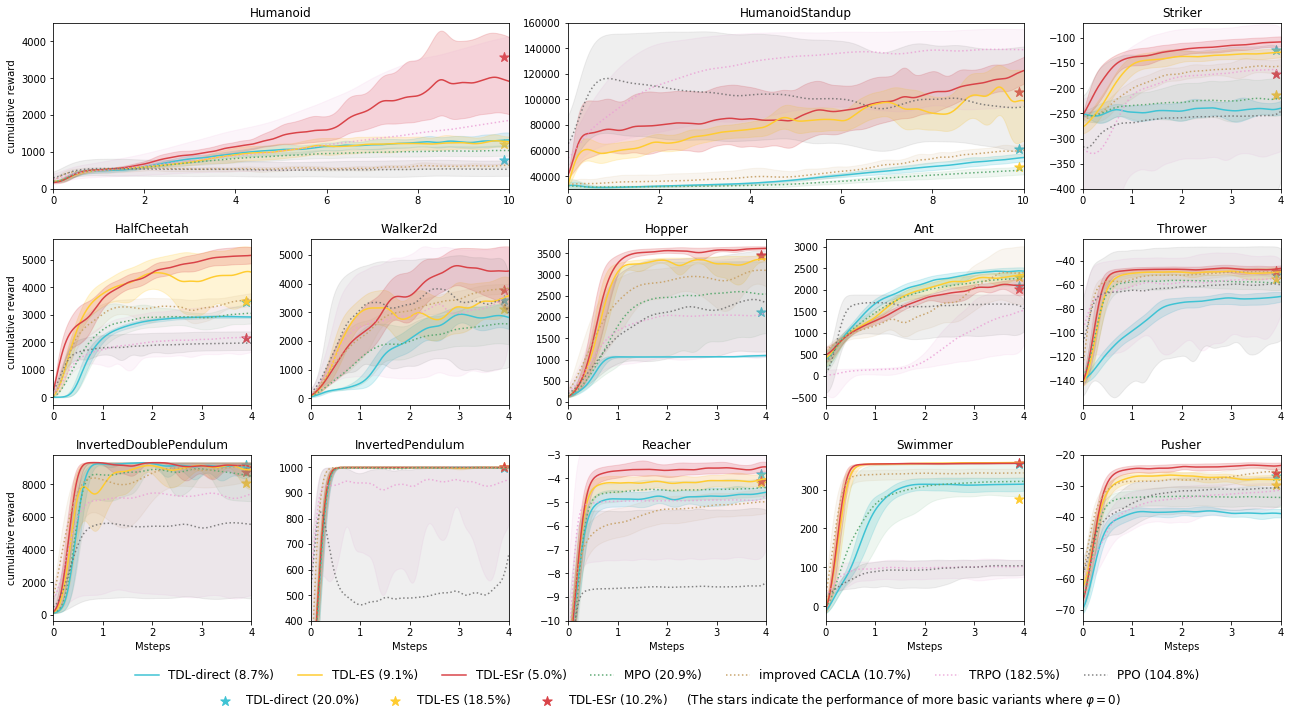}
   \caption{Comparison of several algorithms on MuJoCo tasks. The lines indicate the moving average across five independent runs and the shaded areas indicate the 10\% and 90\% quantiles. The percentage numbers in the legend indicate the normalized fluctuation of the scores in the last 100 iterations averaged over all the tasks.}
   \label{fig:benchmark}
\end{figure*}

\subsection{Performance on continuous control benchmarks}
\label{chap:exp4}

We implemented TDL-direct, TDL-ES and TDL-ESr for the continuous control tasks provided by OpenAI Gym \cite{1606.01540} using MuJoCo simulator \cite{todorov2012mujoco}. 
Due to space limit, the detailed setting of 
hyperparameters can be found 
in Appendix \ref{app:hyperparameters}.
In our experiments, we compare our algorithms against 
TRPO \cite{schulman2015trust} and PPO \cite{schulman2017proximal} (the clipped version) which are two popular policy gradient-based algorithms. 
We also compare our algorithms with MPO \cite{abdolmaleki2018maximum} 
and an improved version of CACLA \cite{van2007reinforcement}. 
Our algorithms differ from MPO in the way we set target distributions. 
Specifically, we set target statistical parameters, whereas MPO sets target probability density values on action samples based on an estimated Q-function.
The improved version of CACLA is actually an ablated version of TDL-ES without setting targets, i.e., the policy network is directly updated by stochastic gradient ascent on $L_{t,2}$.
We show the results in Figure \ref{fig:benchmark}. 

We can see that at least one of our algorithms outperform the previous algorithms in most tasks. 
TDL-ES performs better than the improved CACLA as setting targets prevents destructively large updates. 
Our algorithms generally outperform MPO which illustrates the effectiveness of the way we set target distributions.
Moreover, the performance fluctuation during the training (and especially at the end of the training) for our algorithms is typically small (see also the normalized fluctuation in the legend of Figure \ref{fig:benchmark}), 
which indicates that the training processes of our algorithms are more stable and steadily improved over iterations. 
This also indicates that our algorithms are robust across different seeds.
For tasks that require a precise control such as \texttt{Reacher} and \texttt{InvertedPendulum}, our algorithms result in a higher average cumulative reward than TRPO and PPO.
This is due to the fact that our algorithms address the instability issue illustrated in Section \ref{sec:instability}. 
In particular, TDL-ESr performs the best on most tasks, especially the ones with large action space dimensions (such as \texttt{Humanoid}).

\subsection{On-policy sample reuse}
\label{chap:exp1}

\begin{figure*}[h!]
   \centering
   \includegraphics[width=0.8\textwidth]{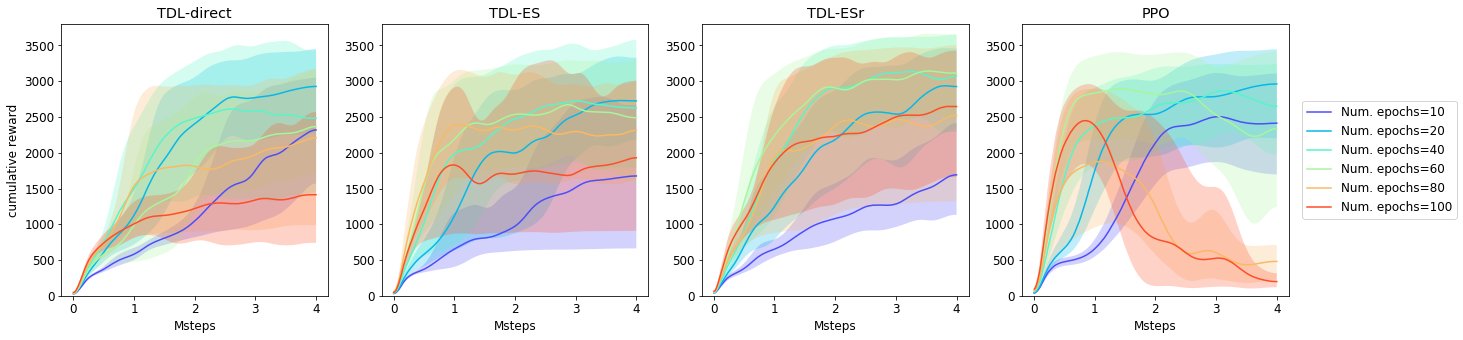}    
   \caption{Performance of TDL-direct, TDL-ES, TDL-ESr and PPO on \texttt{Hopper-v2} task with different levels of sample reuse. \texttt{Num.} \texttt{epochs} denotes the average number of times that a sample is used the for the neural network update. The lines indicate the moving average across five independent runs and the shaded areas indicate the 10\% and 90\% quantiles.}
   \label{fig:epochs}
\end{figure*}

\begin{figure*}[h!]
   \centering
   \includegraphics[width=0.7\textwidth]{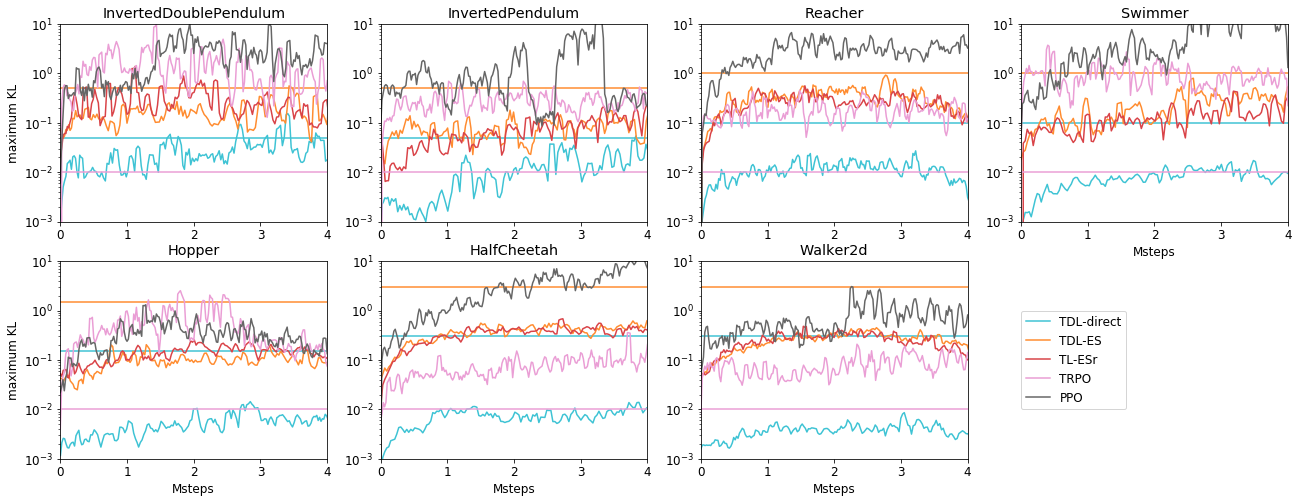}
   \caption{The maximum KL divergence during the training on different tasks. The mean KL constraint in TRPO is set to $\delta=0.01$ and the level is indicated by the straight pink lines (bottom). The maximum KL bounds on samples for TDL-direct and TDL-ES(r) are indicated by the straight blue lines (middle) and the straight orange lines (top) respectively.}
   \label{fig:KL}
\end{figure*}

Unlike off-policy algorithms that use past experiences to improve sample efficiency, on-policy algorithms can improve the sample efficiency by learning more epochs on the same on-policy samples. We compare our algorithms against PPO with different level of sample reuse and show the result in Figure \ref{fig:epochs}. Notice that PPO is quite similar to our algorithms and the main difference is that PPO updates along the policy gradient and ours update to match the target distributions. We see that, in PPO, although the sample efficiency improves from the increase of 
the sample reuse,
the performance gets damaged. In contrast, TDL methods avoid this issue and we can safely increase the sample reuse. This is due to the fact that the policy network in TDL learns to match the fixed target distributions, whereas the policy network in PPO updates along the policy gradient of a clipped surrogate objective. In PPO, when iteratively optimizing this objective, more samples are masked by the clipping and the action distributions conditioned on the corresponding state samples may stray away.

\subsection{KL divergence constraints}
\label{chap:exp2}

Our algorithms, like TRPO and PPO, rely on a conservative policy iteration that requires a constraint in the state-action space. This experiment is designed to evaluate how effective these algorithms can enforce such a constraint. TRPO, PPO and our algorithms aim to constrain the maximum KL divergence across the state space. In this experiment, we approximate the maximum KL divergence across the state space by first sampling a holdout set of $2048$ transitions in each iteration and then recording the maximum KL divergence of the action distributions conditioned on each state sample in the holdout set over the iteration.

We show the results in Figure \ref{fig:KL}. 
First, we observe that in our algorithms the maximum KL divergences are effectively bounded by the limits we set while in TRPO the maximum KL divergences can be up to two orders of magnitude larger than the prescribed limits. Second, the maximum KL divergences in our algorithms (especially in TDL-direct) are generally smaller than those of TRPO and PPO, indicating a more conservative policy update in our algorithms. Thus, our algorithms result in a more stably improved performance during the training while achieving a comparable asymptotic performance and sample efficiency. The result indicates that the sample-wise constraint in TDL is more effective in enforcing a global constraint in the state-action space than previous methods.

\section{Conclusion}

We proposed a new method, called \emph{target distribution learning}, to optimize stochastic policies for continuous control. This method proposes target distributions in each iteration and then trains the policy network to match these distributions. It enables a safe increase in the sample reuse to improve the sample efficiency for an on-policy algorithm. We designed three algorithms via this method. These algorithms can effectively impose constraint on the state-action space and avoid the instability problem of some prior policy gradient-based methods. Empirically, our algorithms achieve comparable performances to some state-of-the-art algorithms on a set of continuous control benchmark tasks.

In this paper, we focus on on-policy algorithms and Gaussian distribution for the action distribution. However, \emph{target distribution learning} can be readily extended to off-policy settings, other types of action distributions and other types of constraints in the state-action space. We leave the extension as an interesting future direction.

\section*{Acknowledgement}

The research is supported in part by the National Natural Science Foundation of China Grant 61822203, 61772297, 61632016, 61761146003, and the Zhongguancun Haihua Institute for Frontier Information Technology and Turing AI Institute of Nanjing.

{
\fontsize{8.9pt}{9.9pt} \selectfont
\bibliographystyle{aaai}
\bibliography{reference}

\begin{thebibliography}{}

\bibitem[\protect\citeauthoryear{Abdolmaleki \bgroup et al\mbox.\egroup
  }{2018a}]{abdolmaleki2018relative}
Abdolmaleki, A.; Springenberg, J.~T.; Degrave, J.; Bohez, S.; Tassa, Y.; Belov,
  D.; Heess, N.; and Riedmiller, M.
\newblock 2018a.
\newblock {Relative Entropy Regularized Policy Iteration}.
\newblock {\em arXiv preprint arXiv:1812.02256}.

\bibitem[\protect\citeauthoryear{Abdolmaleki \bgroup et al\mbox.\egroup
  }{2018b}]{abdolmaleki2018maximum}
Abdolmaleki, A.; Springenberg, J.~T.; Tassa, Y.; Munos, R.; Heess, N.; and
  Riedmiller, M.
\newblock 2018b.
\newblock {Maximum a Posteriori Policy Optimisation}.
\newblock In {\em Proc. 6th ICLR}.

\bibitem[\protect\citeauthoryear{Agarwal \bgroup et al\mbox.\egroup
  }{2019}]{agarwal2019optimality}
Agarwal, A.; Kakade, S.~M.; Lee, J.~D.; and Mahajan, G.
\newblock 2019.
\newblock {Optimality and Approximation with Policy Gradient Methods in Markov
  Decision Processes}.
\newblock {\em arXiv preprint arXiv:1908.00261}.

\bibitem[\protect\citeauthoryear{Beyer and Schwefel}{2002}]{beyer2002evolution}
Beyer, H.-G., and Schwefel, H.-P.
\newblock 2002.
\newblock {Evolution Strategies: A Comprehensive Introduction}.
\newblock {\em Natural Computing} 1(1):3--52.

\bibitem[\protect\citeauthoryear{Brockman \bgroup et al\mbox.\egroup
  }{2016}]{1606.01540}
Brockman, G.; Cheung, V.; Pettersson, L.; Schneider, J.; Schulman, J.; Tang,
  J.; and Zaremba, W.
\newblock 2016.
\newblock {OpenAI Gym}.
\newblock {\em arXiv preprint arXiv:1606.01540}.

\bibitem[\protect\citeauthoryear{Dabney \bgroup et al\mbox.\egroup
  }{2018}]{dabney2018implicit}
Dabney, W.; Ostrovski, G.; Silver, D.; and Munos, R.
\newblock 2018.
\newblock {Implicit Quantile Networks for Distributional Reinforcement
  Learning}.
\newblock {\em arXiv preprint arXiv:1806.06923}.

\bibitem[\protect\citeauthoryear{Gr{\"u}ttner \bgroup et al\mbox.\egroup
  }{2010}]{gruttner2010multi}
Gr{\"u}ttner, M.; Sehnke, F.; Schaul, T.; and Schmidhuber, J.
\newblock 2010.
\newblock {Multi-dimensional Deep Memory Atari-Go Players for Parameter
  Exploring Policy Gradients}.
\newblock In {\em Proc. 19th ICANN},  114--123.

\bibitem[\protect\citeauthoryear{Gu \bgroup et al\mbox.\egroup
  }{2016}]{gu2016continuous}
Gu, S.; Lillicrap, T.; Sutskever, I.; and Levine, S.
\newblock 2016.
\newblock {Continuous Deep Q-learning with Model-based Acceleration}.
\newblock In {\em Proc. 33rd ICML},  2829--2838.

\bibitem[\protect\citeauthoryear{H{\"a}m{\"a}l{\"a}inen \bgroup et
  al\mbox.\egroup }{2018}]{hamalainen2018ppo}
H{\"a}m{\"a}l{\"a}inen, P.; Babadi, A.; Ma, X.; and Lehtinen, J.
\newblock 2018.
\newblock {PPO-CMA: Proximal Policy Optimization with Covariance Matrix
  Adaptation}.
\newblock {\em arXiv preprint arXiv:1810.02541}.

\bibitem[\protect\citeauthoryear{Hansen and
  Ostermeier}{2001}]{hansen2001completely}
Hansen, N., and Ostermeier, A.
\newblock 2001.
\newblock {Completely Derandomized Self-adaptation in Evolution Strategies}.
\newblock {\em Evolutionary Computation} 9(2):159--195.

\bibitem[\protect\citeauthoryear{Hansen}{2000}]{hansen2000invariance}
Hansen, N.
\newblock 2000.
\newblock {Invariance, Self-adaptation and Correlated Mutations in Evolution
  Strategies}.
\newblock In {\em Proc. 6th PPSN},  355--364.
\newblock Springer.

\bibitem[\protect\citeauthoryear{Ilyas \bgroup et al\mbox.\egroup
  }{2018}]{ilyas2018deep}
Ilyas, A.; Engstrom, L.; Santurkar, S.; Tsipras, D.; Janoos, F.; Rudolph, L.;
  and Madry, A.
\newblock 2018.
\newblock {Are Deep Policy Gradient Algorithms Truly Policy Gradient
  Algorithms?}
\newblock {\em arXiv preprint arXiv:1811.02553}.

\bibitem[\protect\citeauthoryear{Kakade and
  Langford}{2002}]{kakade2002approximately}
Kakade, S., and Langford, J.
\newblock 2002.
\newblock {Approximately Optimal Approximate Reinforcement Learning}.
\newblock In {\em Proc. 19th ICML},  267--274.

\bibitem[\protect\citeauthoryear{Kullback and
  Leibler}{1951}]{kullback1951information}
Kullback, S., and Leibler, R.~A.
\newblock 1951.
\newblock {On Information and Sufficiency}.
\newblock {\em The Annals of Mathematical Statistics} 22(1):79--86.

\bibitem[\protect\citeauthoryear{Lange}{2012}]{lange2012potential}
Lange, R.-J.
\newblock 2012.
\newblock {Potential Theory, Path Integrals and the Laplacian of the
  Indicator}.
\newblock {\em Journal of High Energy Physics} 2012(11):32.

\bibitem[\protect\citeauthoryear{Lillicrap \bgroup et al\mbox.\egroup
  }{2015}]{lillicrap2015continuous}
Lillicrap, T.~P.; Hunt, J.~J.; Pritzel, A.; Heess, N.; Erez, T.; Tassa, Y.;
  Silver, D.; and Wierstra, D.
\newblock 2015.
\newblock {Continuous Control with Deep Reinforcement Learning}.
\newblock {\em arXiv preprint arXiv:1509.02971}.

\bibitem[\protect\citeauthoryear{Liu \bgroup et al\mbox.\egroup
  }{2019}]{liu2019trust}
Liu, G.; Zhao, L.; Yang, F.; Bian, J.; Qin, T.; Yu, N.; and Liu, T.-Y.
\newblock 2019.
\newblock {Trust Region Evolution Strategies}.
\newblock In {\em Proc. 33rd AAAI},  4352--4359.

\bibitem[\protect\citeauthoryear{Mania, Guy, and Recht}{2018}]{mania2018simple}
Mania, H.; Guy, A.; and Recht, B.
\newblock 2018.
\newblock {Simple Random Search Provides a Competitive Approach to
  Reinforcement Learning}.
\newblock {\em arXiv preprint arXiv:1803.07055}.

\bibitem[\protect\citeauthoryear{Mnih \bgroup et al\mbox.\egroup
  }{2015}]{mnih2015human}
Mnih, V.; Kavukcuoglu, K.; Silver, D.; Rusu, A.~A.; Veness, J.; Bellemare,
  M.~G.; Graves, A.; Riedmiller, M.; Fidjeland, A.~K.; Ostrovski, G.; et~al.
\newblock 2015.
\newblock {Human-level Control through Deep Reinforcement Learning}.
\newblock {\em Nature} 518(7540):529.

\bibitem[\protect\citeauthoryear{Mnih \bgroup et al\mbox.\egroup
  }{2016}]{mnih2016asynchronous}
Mnih, V.; Badia, A.~P.; Mirza, M.; Graves, A.; Lillicrap, T.; Harley, T.;
  Silver, D.; and Kavukcuoglu, K.
\newblock 2016.
\newblock {Asynchronous Methods for Deep Reinforcement Learning}.
\newblock In {\em Proc. 33rd ICML},  1928--1937.

\bibitem[\protect\citeauthoryear{Novati and
  Koumoutsakos}{2019}]{novati2018remember}
Novati, G., and Koumoutsakos, P.
\newblock 2019.
\newblock {Remember and Forget for Experience Replay}.
\newblock In {\em Proc. 36th ICML},  4851--4860.

\bibitem[\protect\citeauthoryear{Pirotta \bgroup et al\mbox.\egroup
  }{2013}]{pirotta2013safe}
Pirotta, M.; Restelli, M.; Pecorino, A.; and Calandriello, D.
\newblock 2013.
\newblock {Safe Policy Iteration}.
\newblock In {\em Proc. 30th ICML},  307--315.

\bibitem[\protect\citeauthoryear{Rechenberg}{1973}]{rechenberg1973evolutionsstrategie}
Rechenberg, I.
\newblock 1973.
\newblock {Evolutionsstrategie--Optimierung Technischer Systeme nach Prinzipien
  der Biologischen Information}.
\newblock {\em Stuttgart-Bad Cannstatt: Friedrich Frommann Verlag}.

\bibitem[\protect\citeauthoryear{Salimans \bgroup et al\mbox.\egroup
  }{2017}]{salimans2017evolution}
Salimans, T.; Ho, J.; Chen, X.; Sidor, S.; and Sutskever, I.
\newblock 2017.
\newblock {Evolution Strategies as a Scalable Alternative to Reinforcement
  Learning}.
\newblock {\em arXiv preprint arXiv:1703.03864}.

\bibitem[\protect\citeauthoryear{Schaul \bgroup et al\mbox.\egroup
  }{2015}]{schaul2015prioritized}
Schaul, T.; Quan, J.; Antonoglou, I.; and Silver, D.
\newblock 2015.
\newblock {Prioritized Experience Replay}.
\newblock {\em arXiv preprint arXiv:1511.05952}.

\bibitem[\protect\citeauthoryear{Scherrer}{2014}]{scherrer2014approximate}
Scherrer, B.
\newblock 2014.
\newblock {Approximate Policy Iteration Schemes: A Comparison}.
\newblock In {\em Proc. 31st ICML},  1314--1322.

\bibitem[\protect\citeauthoryear{Schulman \bgroup et al\mbox.\egroup
  }{2015a}]{schulman2015trust}
Schulman, J.; Levine, S.; Abbeel, P.; Jordan, M.~I.; and Moritz, P.
\newblock 2015a.
\newblock {Trust Region Policy Optimization}.
\newblock In {\em Proc. 32nd ICML}, volume~37,  1889--1897.

\bibitem[\protect\citeauthoryear{Schulman \bgroup et al\mbox.\egroup
  }{2015b}]{schulman2015high}
Schulman, J.; Moritz, P.; Levine, S.; Jordan, M.; and Abbeel, P.
\newblock 2015b.
\newblock {High-dimensional Continuous Control using Generalized Advantage
  Estimation}.
\newblock {\em arXiv preprint arXiv:1506.02438}.

\bibitem[\protect\citeauthoryear{Schulman \bgroup et al\mbox.\egroup
  }{2017}]{schulman2017proximal}
Schulman, J.; Wolski, F.; Dhariwal, P.; Radford, A.; and Klimov, O.
\newblock 2017.
\newblock {Proximal Policy Optimization Algorithms}.
\newblock {\em arXiv preprint arXiv:1707.06347}.

\bibitem[\protect\citeauthoryear{Silver \bgroup et al\mbox.\egroup
  }{2014}]{silver2014deterministic}
Silver, D.; Lever, G.; Heess, N.; Degris, T.; Wierstra, D.; and Riedmiller, M.
\newblock 2014.
\newblock {Deterministic Policy Gradient Algorithms}.
\newblock In {\em Proc. 31st ICML},  387--395.

\bibitem[\protect\citeauthoryear{Silver \bgroup et al\mbox.\egroup
  }{2017}]{silver2017mastering}
Silver, D.; Schrittwieser, J.; Simonyan, K.; Antonoglou, I.; Huang, A.; Guez,
  A.; Hubert, T.; Baker, L.; Lai, M.; Bolton, A.; et~al.
\newblock 2017.
\newblock {Mastering the Game of Go without Human Knowledge}.
\newblock {\em Nature} 550(7676):354.

\bibitem[\protect\citeauthoryear{Sutton, Barto, and
  others}{1998}]{sutton1998introduction}
Sutton, R.~S.; Barto, A.~G.; et~al.
\newblock 1998.
\newblock {\em {Introduction to Reinforcement Learning}}, volume 135.
\newblock MIT press Cambridge.

\bibitem[\protect\citeauthoryear{Todorov, Erez, and
  Tassa}{2012}]{todorov2012mujoco}
Todorov, E.; Erez, T.; and Tassa, Y.
\newblock 2012.
\newblock {MuJoCo: A Physics Engine for Model-based Control}.
\newblock In {\em 2012 IEEE/RSJ International Conference on Intelligent Robots
  and Systems},  5026--5033.
\newblock IEEE.

\bibitem[\protect\citeauthoryear{Van~Hasselt and
  Wiering}{2007}]{van2007reinforcement}
Van~Hasselt, H., and Wiering, M.~A.
\newblock 2007.
\newblock {Reinforcement Learning in Continuous Action Spaces}.
\newblock In {\em 2007 IEEE International Symposium on Approximate Dynamic
  Programming and Reinforcement Learning},  272--279.
\newblock IEEE.

\bibitem[\protect\citeauthoryear{Van~Hasselt, Guez, and
  Silver}{2016}]{van2016deep}
Van~Hasselt, H.; Guez, A.; and Silver, D.
\newblock 2016.
\newblock {Deep Reinforcement Learning with Double Q-learning}.
\newblock In {\em Proc. 30th AAAI},  2094--2100.

\bibitem[\protect\citeauthoryear{Wang \bgroup et al\mbox.\egroup
  }{2016}]{wang2015dueling}
Wang, Z.; Schaul, T.; Hessel, M.; Van~Hasselt, H.; Lanctot, M.; and De~Freitas,
  N.
\newblock 2016.
\newblock {Dueling Network Architectures for Deep Reinforcement Learning}.
\newblock In {\em Proc. 33rd ICML},  1995--2003.

\bibitem[\protect\citeauthoryear{Wang \bgroup et al\mbox.\egroup
  }{2017}]{wang2016sample}
Wang, Z.; Bapst, V.; Heess, N.; Mnih, V.; Munos, R.; Kavukcuoglu, K.; and
  de~Freitas, N.
\newblock 2017.
\newblock {Sample Efficient Actor-Critic with Experience Replay}.
\newblock In {\em Proc. 5th ICLR}.

\bibitem[\protect\citeauthoryear{Wierstra \bgroup et al\mbox.\egroup
  }{2014}]{wierstra2014natural}
Wierstra, D.; Schaul, T.; Glasmachers, T.; Sun, Y.; Peters, J.; and
  Schmidhuber, J.
\newblock 2014.
\newblock {Natural Evolution Strategies}.
\newblock {\em The Journal of Machine Learning Research} 15(1):949--980.

\bibitem[\protect\citeauthoryear{Williams}{1992}]{williams1992simple}
Williams, R.~J.
\newblock 1992.
\newblock {Simple Statistical Gradient-following Algorithms for Connectionist
  Reinforcement Learning}.
\newblock {\em Machine learning} 8(3-4):229--256.

\bibitem[\protect\citeauthoryear{Wu \bgroup et al\mbox.\egroup
  }{2017}]{wu2017scalable}
Wu, Y.; Mansimov, E.; Grosse, R.~B.; Liao, S.; and Ba, J.
\newblock 2017.
\newblock {Scalable Trust-region Method for Deep Reinforcement Learning using
  Kronecker-factored Approximation}.
\newblock In {\em Proc. 31st NeurIPS},  5279--5288.

\bibitem[\protect\citeauthoryear{Zhao \bgroup et al\mbox.\egroup
  }{2011}]{zhao2011analysis}
Zhao, T.; Hachiya, H.; Niu, G.; and Sugiyama, M.
\newblock 2011.
\newblock {Analysis and Improvement of Policy Gradient Estimation}.
\newblock In {\em Proc. 25th NeurIPS},  262--270.

\end{thebibliography}
}

\clearpage

\begin{alphasection}


\section{The KL Divergence Constraint}
\label{app:KL_bound}

In this section, we show that the proposed target means in TDL-direct and TDL-ES satisfy the KL divergence constraint.
In other words, the KL divergence between the old and the new action distributions conditioned on each state sample is bounded, if the mean of the new action distribution is specified by the target mean proposed in the algorithms.

Conditioned on a state sample $s_t$, we consider the target mean $\hat{\mu}_t$ and the new standard deviation of the action distribution $\sigma_\theta(s_t)$. 
For simplicity, 
instead of using $(\hat{\mu}_t, \sigma_\theta(s_t))$, we use the relative values $(\mu_t, \sigma_t)$ to specify the new action distribution. 
More precisely, their relationship is as follows: $\hat{\mu}_t = \mu^{\text{old}}(s_t) + \mu_t \sigma^{\text{old}}(s_t), \sigma_\theta(s_t) = \sigma_t \sigma^{\text{old}}(s_t)$.
Given that the probability distribution of the policy is the multivariate Gaussian distribution with diagonal covariance matrix, the KL divergence between the old and the new action distributions conditioned on one sample can be written in terms of the old and the new statistical parameters conditioned on the corresponding state.

\begin{equation}
\begin{aligned}
\text{KL}_t & := \text{KL}(\mathcal{N}(\cdot|\mu^{\text{old}}(s_t), \sigma^{\text{old}}(s_t))||\mathcal{N}(\cdot|\hat{\mu}_t, {\sigma}_\theta(s_t))) \\
& = \dfrac{1}{2} \sum_{i=1}^d \left[ 2 \log \sigma_{ti} + \dfrac{1 + {\mu_{ti}}^2}{{\sigma_{ti}}^2} - 1   \right]
\end{aligned}{}
\end{equation}

where the subscript $i$ indicates the $i$-th element in the vector and $d$ is the dimension of the action space. Notice that the standard deviation of the action distribution contains a state independent component and it typically does not change too much over one iteration. Hence, we assume $\sigma_{ti} \in [1-\epsilon, 1+\epsilon]$ with a small value $\epsilon$. 

For TDL-direct, recall that 
$$\mu_{ti}^2 \le \left( \min(1, \dfrac{\sqrt{2\alpha}}{||y_t||_2}) y_t \right)^2 \le 2\alpha ,$$
where the action sample $a_t$ is obtained from $a_t = \mu^{\text{old}}(s_t) + y_t \sigma^{\text{old}}$ by sampling a $y_t \sim \mathcal{N}(0, I)$. Therefore, we can see that 
$\text{KL}_t$ is bounded, as desired.

\begin{eqnarray}
\nonumber
\text{KL}_t \le & \dfrac{1}{2} \sum_{i=1}^d \left[ 2 \log (1-\epsilon) + \dfrac{1 + 2\alpha}{(1-\epsilon)^2} - 1  \right] \\
= & d \alpha (1+2\epsilon) + o(\epsilon^2) \approx d\alpha
\end{eqnarray}

For TDL-ES, $\mathbb{E}_{y_{ti}} [ \mu_{ti}^2 ] \le \nu^2 \mathbb{E}_{y_{ti}}[y_{ti}^2] = \nu^2$. Therefore, the expected $\text{KL}_t$ can be bounded as follows:

\begin{eqnarray}
\nonumber
\mathbb{E}_{a_t}[\text{KL}_t] \le & \dfrac{1}{2} \sum_{i=1}^d \left[ 2 \log (1-\epsilon) + \dfrac{1 + \nu^2}{(1-\epsilon)^2} - 1  \right] \\
= & \dfrac{1}{2} d \nu^2 (1+2\epsilon) + o(\epsilon^2) \approx \dfrac{1}{2} d \nu^2
\end{eqnarray}


Notice that the policy improvement theorem requires
a worst-case KL bound over all possible states \cite{ilyas2018deep}.
TDL-direct and TDL-ES enforces the KL constraint for each state sample. In contrast, TRPO bounds the mean KL based on the states observed in the iteration and PPO disincentivizes the update that violates the constraint. Our previous experiment in Figure \ref{fig:KL} verifies that the constraint enforced by our algorithms is closest to that required by the policy improvement theorem, compared with TRPO and PPO.


\section{Target Distribution for TDL-ES as a Policy Improvement Step}
\label{app:ES_theory}

Policy iteration alternates between the policy evaluation step and the policy improvement step. In each iteration, the policy improvement step in our algorithms aims to propose target statistical parameters that optimize a constrained surrogate objective and then update the statistical parameters to the targets.



In this section, we study the behavior of TDL-ES in detail, in order to show the following two properties.
First, the target statistical parameters proposed in TDL-ES can be regarded as a gradient ascent step with respect to the probability that the target policy is improved over the old value function (i.e., $L_{t,2}(\mu, \sigma)$).
Second, by assuming the critic is fixed and solving for the fixed point of the action distribution, we show that the policy in TDL-ES tends to approach a conservative stochastic policy, whereas vanilla policy gradient (and other previous methods) tends to approach a greedy deterministic policy. This prevents a premature convergence when the critic is not well learnt which may result from the poor performance of the policy.


\subsection{Targets as a stochastic gradient ascent step}

When updating the target variance, we involve a state independent component to avoid the violation of the constraint. Here, to better analyze the property of the proposed target statistical parameters, we ignore the state independent component which is the average of $\hat{\sigma}_t$ on all the samples in one iteration. 
Hence, the updates of the action distributions conditioned on different states are independent and we only need to analyze the case on one state.
In another word, given a state $s_t$, we analyze the proposed target statistical parameters $\hat{\sigma}_t$ and $\hat{\mu}_t$ (cf. \eqref{eq:state_dependent_variance} and \eqref{eq:mean_ES}). 

We assume the advantage function $A^{\pi^{\text{old}}}(s_t, a)$ is given by the difference of the action value function $Q^{\pi^{\text{old}}}(s_t, a)$ and the state value $V^{\pi^{\text{old}}}(s_t)$ instead of GAE used in real algorithms. Since the analysis applies to different states independently, for simplicity, we use $Q(a)$ as shorthand for $Q^{\pi^{\text{old}}}(s_t, a)$ and ${V}$ as shorthand for $V^{\pi^{\text{old}}}(s_t)$.



Before we present the result, let us give several definitions. We define $F(a) := \mathbb{I} \{Q(a) - {V} > 0\}$ and $G(\mu, \sigma) := \mathbb{E}_{y \sim \mathcal{N}(0, I)} \left[ F(\mu + y \sigma) \right]$, where $\mathbb{I}(\cdot)$ is the indicator function. The function $F$ is an indicator of whether the advantage is positive or negative. The function $G$ is obtained from $F$ by smoothing over $F$ with a Gaussian filter whose standard deviation is $\sigma$, which is a shorthand for $L_{t,2}$. We illustrate their relationship in Figure \ref{fig:QFG}. 

\begin{figure}[htbp]
	\centering
	\includegraphics[width=\columnwidth]{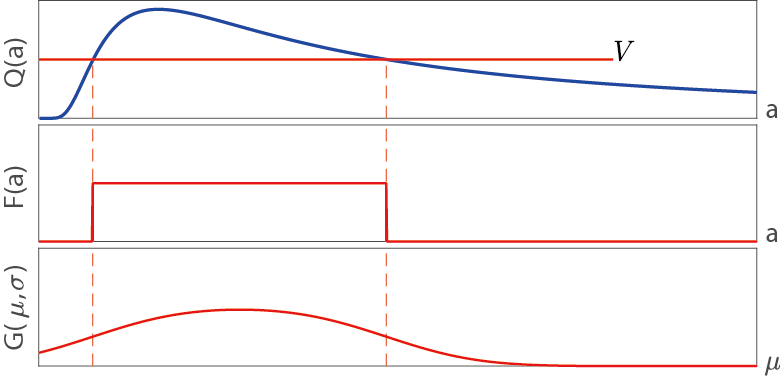}
	\caption{The functions $Q(a)$, $F(a)$ and $G(\mu, \sigma)$ for a given $\sigma$.}
	\label{fig:QFG}
\end{figure}

Suppose the old action distribution on this state is $\mathcal{N}(\mu, \sigma)$. The proposed target statistical parameters $(\hat{\mu}_t, \hat{\sigma}_t)$ are Monte Carlo estimates of $(\mu', \sigma')$, where

\begin{equation}
\label{eq:expected_target}
\begin{aligned}
\mu' & = \mathbb{E}_{a \sim \mathcal{N}(\mu, \sigma)} \Big[ \mu \mathbb{I} \{ Q(a) \le {V} \} \\
& \quad \quad + ( \mu + \nu (a - \mu) ) \mathbb{I} \{ Q(a) > {V} \}  \Big] \\
(\sigma')^2 & = \mathbb{E}_{a \sim \mathcal{N}(\mu, \sigma)} \Big[ \sigma^2 \mathbb{I} \{ Q(a) \le {V} \} \\
& \quad \quad + (a - \mu)^2 \mathbb{I} \{ Q(a) > {V} \} \Big]
\end{aligned}
\end{equation}

Our result is that the target statistical parameters proposed in TDL-ES is actually a stochastic gradient ascent step with respect to $G$. 

\begin{theorem}
\label{theorem1}
The target statistical parameters $(\hat{\mu}_t, \hat{\sigma_t})$ defined in \eqref{eq:state_dependent_variance} and \eqref{eq:mean_ES} are Monte Carlo estimates of $(\mu', \sigma')$ defined in \eqref{eq:expected_target}. $(\mu', \sigma')$ is a gradient ascent update of $(\mu, \sigma)$ with respect to $G(\mu, \sigma)$. Specifically,
\begin{align}
\label{eq:one_step_update}
    \begin{cases}
    \mu' & = \mu + \nu \sigma^2 \nabla_\mu G(\mu, \sigma) \\
    (\sigma')^2 & = \sigma^2 + 2 \sigma^4 \nabla_{\sigma^2} G(\mu, \sigma) \\
    \end{cases}.
\end{align}
\end{theorem}

From this point on, we use $\mathbb{E}$ to refer $\mathbb{E}_{y \sim \mathcal{N}(0, I)}$.

\begin{lemma}
\label{lemma1}
For vectors $\mu, \sigma, y \in \mathbb{R}^d$ and a function $f:\mathbb{R}^d \to \mathbb{R}$,
$$
\nabla_\mu \mathbb{E} \left[ f(\mu + y\sigma) \right] = \dfrac{1}{\sigma} \mathbb{E} \left[ y f(\mu + y\sigma) \right],
$$
$$
\nabla_{\sigma^2} \mathbb{E} \left[ f(\mu + y\sigma) \right] = \dfrac{1}{2\sigma^2} \mathbb{E} \left[ (y^2 - 1) f(\mu + y\sigma) \right].
$$
\end{lemma}

The first equation in the lemma is also used in \cite{salimans2017evolution}.

\begin{proof}[Proof of Theorem \ref{theorem1}]
The theorem can be proved directly using Lemma \ref{lemma1}.

\begin{align}
\label{eq:mu_update}
\mu' -  \mu 
& = \nu \sigma \mathbb{E} \left[ y F(\mu + y\sigma) \right] \nonumber\\
& = \nu \sigma^2 \nabla_\mu G(\mu, \sigma)
\end{align}

\begin{align}
\label{eq:sigma_update}
\dfrac{\sigma'^2 - \sigma^2}{\sigma^2} & = \mathbb{E} \left[ (y^2 - 1) F(\mu + y \sigma) \right] \nonumber\\
& = 2 \sigma^2 \nabla_{\sigma^2} G(\mu, \sigma) 
\end{align}

\end{proof}

\begin{figure}[htbp]
   \centering
   \includegraphics[width=\columnwidth]{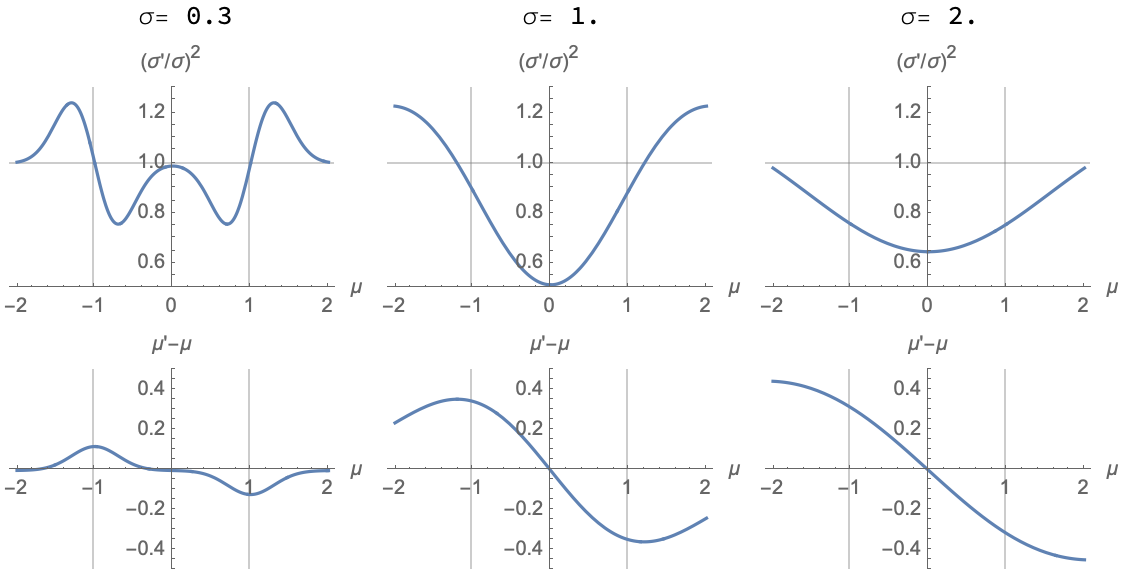}    
   \caption{The target statistical parameters varies with respect to the old statistical parameters when $Q(a)=-a^2$ and ${V}=-1$. The top row presents the value of the target variance $(\sigma')^2$ (relative to the old variance $\sigma^2$) and the bottom row presents the value of the target mean $\mu'$ (relative to the old mean $\mu$). The figures in the three columns show different situations when the old standard deviation $\sigma$ varies.}
   \label{fig:theory}
\end{figure}

Figure \ref{fig:theory} shows how the values of the target statistical parameters $(\mu', \sigma')$ varies with respect to different $\mu$ and $\sigma$. We show three cases with different $\sigma$ which stands for different tendency of exploration in the figure from left to right. It can be seen that the target standard deviation is automatically adapted in different situations and the target mean is always approaching the optimal.

\subsection{Tendency of the action distribution update}

To analyze the tendency of the action distribution update in TDL-ES, we consider a virtual process where the action distribution is iteratively updated following the rule of TDL-ES, instead of updating only once. This helps us to understand where the action distribution will be updated to. The process alternates between sampling a sufficient number of action samples from the current action distribution, querying the value of a fixed value function and updating the action distribution following the update rule of TDL-ES. 

Consider only one state as above, and assume the critic is fixed in the process. Define $D:=\{a|Q(a) > {V}\}$ which represents the \say{good} areas in the action space indicated by the critic.
We show that, by iteratively updating following the rule of TDL-ES, the action distribution will approach a distribution that roughly spans over $D$. This indicates that the target mean proposed by TDL-ES approaches the  \say{center} of $D$ and the target standard deviation approaches the (truncated) \say{radius} of $D$.

\begin{theorem}
\label{theorem}
Consider a fixed state value function ${V}$ and a fixed action value function $Q(a)$. The action distribution is iteratively updated following the update rule of TDL-ES as follows:

\begin{equation}
\begin{aligned}
\mu^{(k+1)} & = \mathbb{E}_{a \sim \mathcal{N}(\mu^{(k)}, \sigma^{(k)})} \Big[ \mu^{(k)} \mathbb{I} \{ Q(a) \le {V} \} \\
& \quad \quad + ( \mu^{(k)} + \nu (a - \mu^{(k)}) ) \mathbb{I} \{ Q(a) > {V} \} \Big] \\
(\sigma^{(k+1)})^2 & = \mathbb{E}_{a \sim \mathcal{N}(\mu^{(k)}, \sigma^{(k)})} \Big[ (\sigma^{(k)})^2 \mathbb{I} \{ Q(a) \le {V} \} \\
& \quad \quad + (a - \mu^{(k)})^2 \mathbb{I} \{ Q(a) > {V} \} \Big]
\end{aligned}
\end{equation}


The fixed point of the action distribution $(\mu, \sigma)$ satisfies the following equations:
\begin{align}
\label{eq:fixedpoint}
    \begin{cases}
    \mu & = \dfrac{\int_{x \in D} xf(x)\mathrm{d}x}{\int_{x \in D} f(x) \mathrm{d}x} \\
    \sigma^2 & = \dfrac{\int_{x \in D} (x-\mu)^2 f(x)\mathrm{d}x}{\int_{x \in D} f(x) \mathrm{d}x} \\
    \end{cases},
\end{align}
where $D = \{ x|Q(x)>{V} \}$ and $f(x) = \exp\left(-\frac{(x-\mu)^2}{2\sigma^2}\right)$.
\end{theorem}

Notice that $f(x) \in (0, 1]$ and the dependency on $\mu$ and $\sigma$ is omitted for briefness. We observe that, upon the fixed point, $\mu$ is a weighted average of the points within $D$ and $\sigma^2$ is a weighted average of $(x-\mu)^2$ for $x \in D$. 

Consider a vanilla policy gradient method under such a setting. 
The fixed point of the action distribution for the method is an action distribution that concentrates on $arg \max_a Q(a)$. In other words, the vanilla policy gradient method has a tendency to approach a deterministic policy regardless of the current state value estimate (i.e., $V$) from the critic network. 
However, the update of TDL-ES is aware of the state value estimate which approximates the performance of the current policy. When the performance is still poor currently, the critic will yield a relatively low state value, resulting in a large region for $D$. Therefore, the variance of the action distribution at the fixed point takes a large value which prevents premature convergence. Moreover, when the critic is not well learned and produces overly large or small values (which is quite common at the start of training) resulting in a very small or very large region for $D$, the variance remains the same in expectation due to the lack of information.

At last, we present two lemmas and use them to prove Theorem \ref{theorem}.

\begin{lemma}[Gradient of indicator function \cite{lange2012potential}] 
\label{lemma2}
For a function $f:\mathbb{R}^d \to \mathbb{R}$ that is continuous and differentiable and $D\subseteq \mathbb{R}^d$, 
$$
    \int_{x \in \mathbb{R}^d} f(x) \nabla_x \mathbb{I}\{x \in D \} dx = - \int_{x \in D} \nabla_x f(x) dx.
$$
\end{lemma}


\begin{lemma}[Laplacian of indicator function \cite{lange2012potential}] 
\label{lemma3}
For any continuous and differentiable function $f:\mathbb{R}^d \to \mathbb{R}$ that vanishes at infinity, i.e., $\lim\limits_{\left \| x \right \| \to \infty} f(x) = 0$, and $D\subseteq \mathbb{R}^d$,
$$
    \int_{x \in \mathbb{R}^d} f(x) \nabla^2_x \mathbb{I}\{x \in D \} dx = \int_{x \in D} \nabla^2_x f(x)\mathrm dx.
$$
\end{lemma}

\begin{proof}[Proof of Theorem \ref{theorem}]

Consider an update step $(\mu, \sigma) \to (\mu', \sigma')$. By Theorem \ref{theorem1}, we know that 

\begin{align*}
\mu' - \mu & = \nu \sigma^2 \nabla_\mu G(\mu, \sigma).
\end{align*}

For the variance update, using Lemma \ref{lemma1},

\begin{align*}
& \dfrac{(\sigma')^2 - \sigma^2}{\sigma^2} \nonumber \\
= & \mathbb{E} \left[ (y^2 - 1) F(\mu + y \sigma) \right] \nonumber \\
= & \mathbb{E} \left[ \frac{1}{\sigma} y (\mu + y\sigma) F(\mu + y\sigma) - (\frac{\mu}{\sigma} y + 1) F(\mu + y\sigma) \right] \nonumber \\
= & \nabla_\mu \mathbb{E} \left[ (\mu + y\sigma) F(\mu + y\sigma) \right] \nonumber \\
& \quad \quad - \mu \nabla_\mu \mathbb{E} \left[ F(\mu + y\sigma) \right] - \mathbb{E}[F(\mu + y\sigma)] \nonumber \\ 
= & \mathbb{E} \left[ y\sigma \nabla_\mu F(\mu + y\sigma) \right] \nonumber \\
= & \sigma^2 \nabla_\mu^2 G(\mu, \sigma). \nonumber \\
\end{align*}

Using Lemma \ref{lemma2}, 

\begin{align*}
\label{eq:Gderivative}
\nabla_\mu G(\mu, \sigma)|_\mu & = \nabla_\mu \int_{y \in \mathbb{R}^d} p(y) F(\mu + y\sigma) dy \big|_\mu \nonumber\\
& = \int_{y \in \mathbb{R}^d} p(y) \nabla_\mu F(\mu + y\sigma) \big|_\mu dy \nonumber\\
& = \int_{y \in \mathbb{R}^d} p(y) \nabla_x F(x) \big|_{x=\mu + y\sigma} dy \nonumber\\
& = \int_{x \in \mathbb{R}^d} p\left(\frac{x-\mu}{\sigma}\right) \nabla_x F(x) \frac{1}{\sigma} dx \nonumber\\
& = \frac{1}{(2\pi)^{\frac{d}{2}}\sigma^2} \int_{x \in \mathbb{R}^d} f(x) \nabla_x \mathbb{I} \{x \in D\} dx \nonumber\\
& = - \frac{1}{(2\pi)^{\frac{d}{2}}\sigma^2} \int_{x \in D} \nabla_x f(x) dx, \nonumber\\
\end{align*}
where $f(x) = \exp(-\frac{(x-\mu)^2}{2\sigma^2})$. Notice that $f(x)$ is dependent on $\mu$ and $\sigma$, which are omitted for simplicity.

Using Lemma \ref{lemma3}, 

\begin{align*}
\nabla^2_\mu G(\mu, \sigma)|_\mu & = \int_{y \in \mathbb{R}^d} p(y) \nabla^2_\mu F(\mu + y\sigma)\big|_\mu dy \\
	& = \frac{1}{\sigma}\int_{x \in \mathbb{R}^d} p\left(\frac{x-\mu}{\sigma}\right) \nabla^2_x F(x) dx \\
	& = \frac{1}{(2\pi)^{\frac{d}{2}}\sigma^2} \int_{x \in \mathbb{R}^d} f(x) \nabla^2_x \mathbb{I} \{x \in D\} dx \\
	& = \frac{1}{(2\pi)^{\frac{d}{2}}\sigma^2} \int_{x \in D} \nabla^2_x f(x) dx.
\end{align*}

Upon fixed point, 

$$
\mu' -  \mu = - \frac{\tau}{(2\pi)^{\frac{d}{2}}} \int_{x \in D} \nabla_x f(x) dx = 0,
$$

$$
\dfrac{(\sigma')^2 - \sigma^2}{\sigma^2} = \frac{1}{(2\pi)^{\frac{d}{2}}} \int_{x \in D} \nabla^2_x f(x) dx = 0.
$$

The fixed point can be obtained by substituting $\nabla_x f(x) = -\frac{(x-\mu)}{\sigma^2}f(x)$ and $\nabla^2_x f(x) = \frac{(x-\mu)^2}{\sigma^4}f(x)$ into the above equations.

\begin{align*}
\left\{
\begin{aligned}
\mu \int_{x \in D}f(x) dx  &= \int_{x \in D} xf(x) dx \\
\sigma^2\int_{x \in D} f(x) dx & = \int_{x \in D} (x-\mu)^2 f(x) dx
\end{aligned}
\right.
\end{align*}

Then we can obtain the fixed point (\ref{eq:fixedpoint}), which completes the proof.

\end{proof}


\section{Variance Adaptation}

In this section, we present that the variance of the action distribution is self-adapted in real tasks under the variance update rule of our algorithms. We plot the variance of the action distribution with respect to the training steps in Figure \ref{fig:variance_adaptation}. The results correspond to the experiments shown in Figure \ref{fig:benchmark}. We observe that the variance drops quickly in tasks where our algorithms achieve a good performance, such as \texttt{Reacher} and \texttt{Swimmer}. In contrast, the variance drops slowly or even increases in the tasks where some trouble might be encountered during the learning process, for example plunging into a local minimum. This phenomenon is especially obvious for TDL-direct in \texttt{Hopper} where the algorithm gets stuck but tries to increase the variance to escape. Similarly, comparing the results of TDL-ES and TDL-ESr in \texttt{HumanoidStandup}, the learning curve of TDL-ES is beneath that of TDL-ESr throughout the training while the standard deviation of TDL-ES keeps higher and more fluctuating than TDL-ESr to enlarge exploration.

\begin{figure*}[htbp]
   \centering
   \includegraphics[width=0.8\textwidth]{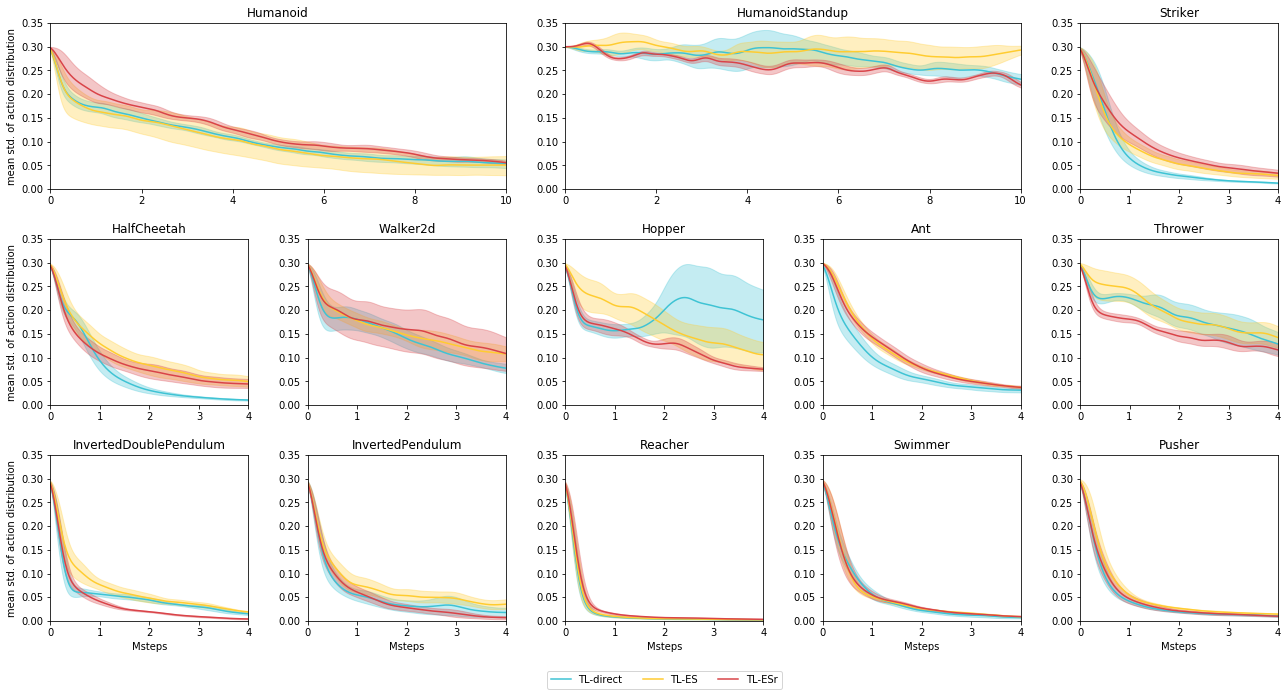}    
   \caption{The standard deviation of the action distribution with respect to the training steps for our algorithms. The shaded areas indicate the areas between the 10\% and 90\% quantiles of the mean standard deviations across five independent runs. This shows that the variance of the action distribution is self-adapted in real tasks under the variance update rule of our algorithms.}
   \label{fig:variance_adaptation}
\end{figure*}

\section{Samples with Negative Advantages}
\label{sec:neg_adv}

Since in most environments the state is not resettable, only one action sample $a_t$ can be obtained from an encountered state $s_t$. An estimator of the advantage $\hat{A}_t$ can be calculated from a critic network and the rewards following $a_t$. This estimator is used to evaluate the action $a_t$. 

In several existing policy gradient methods (such as TRPO and PPO), when $\hat{A}_t > 0$, the stochastic gradient ascent on $\hat{L}_t(\theta)$ updates $\theta$ such that $\mu_\theta(s_t)$ becomes closer to the \say{good} sample $a_t$, which is as expected. When $\hat{A}_t < 0$, it pushes $\mu_\theta(s_t)$ to somewhere along the negative direction, i.e., $\mu^{\text{old}}(s_t) - \lambda a_t$ for some $\lambda > 0$. This assumes the linearity around the mean of the old action distribution \cite{hamalainen2018ppo}. If the assumption holds, this may preserve the useful information and help exploration. Otherwise, pushing the action distribution to such an unexplored direction may cause instability.


In TDL method, we can specify whether to set the target mean to the negative direction or to the mean of the old policy when the advantage estimate is negative. In the algorithms that we propose in the paper, TDL-direct sets the target mean to the negative direction when facing a negative advantage estimate, whereas TDL-ES sets the target mean to the mean of the old policy in such a situation. This choice is based on the following experiments.

We compare variants of TDL-direct and TDL-ES. The target mean for TDL-direct and TDL-ES can be written respectively as follows:

\vspace{-0.7em}
\begin{equation}
\hat{\mu}_t = \mu^{\text{old}}(s_t) + f(\hat{A}_t) \min (1, \dfrac{\sqrt{2\alpha}}{||y_t||_2}) y_t \sigma^{\text{old}}, 
\end{equation}

\begin{equation}
\hat{\mu}_t = \mu^{\text{old}}(s_t) +  f(\hat{A}_t) \nu (a_t - \mu^{\text{old}}(s_t)).
\end{equation}

For the \texttt{neg} variant, we use $f(\hat{A}_t) = \text{sign} (\hat{A}_t)$. For the \texttt{old} variant, we use $f(\hat{A}_t) = \mathbb{I}\{ \hat{A}_t > 0 \}$. 

We show the comparison with these variants in Figure \ref{fig:ablation} and observe the following. First, when the policy updates are conservative (such as TDL-direct), updating the mean of the action distribution to the negative direction helps exploration. However, in \texttt{Humanoid}, the \texttt{neg} variant is not as good as the \texttt{old} variant. We guess that, when the action space dimension is large, a sample with a negative advantage estimate does not indicate the opposite direction is a good direction for exploration. Moreover, this may lead to more instability. Second, when the policy are exploratory (such as TDL-ES), updating the mean of the action distribution to the negative direction leads to significant instability.


\begin{figure*}[htbp]
   \centering
   \includegraphics[width=0.8\textwidth]{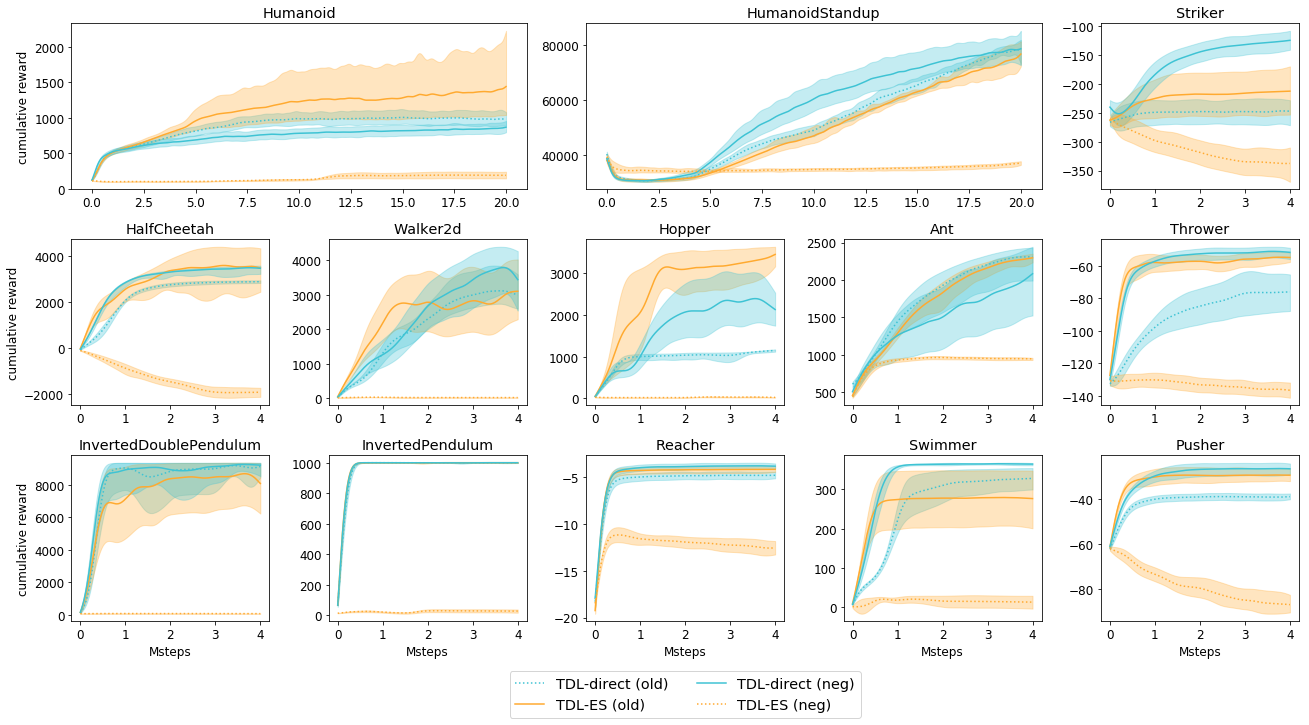}    
   \caption{TDL-ES and TDL-direct compared with the ablated variants. The shaded areas indicate the areas between the 10\% and 90\% quantiles of cumulative rewards across five independent runs. In this set of experiments, we use $\varphi = 0$.}
   \label{fig:ablation}
\end{figure*}

\section{Sensitivity Analysis}

Empirically, we find that the algorithms are most sensitive to the hyperparameters that control the extent to which the policy is allowed to change over iterations. Now, we do the sensitivity analysis for these key hyperparameters.

\begin{figure*}[htbp]
   \centering
   \includegraphics[width=0.5\textwidth]{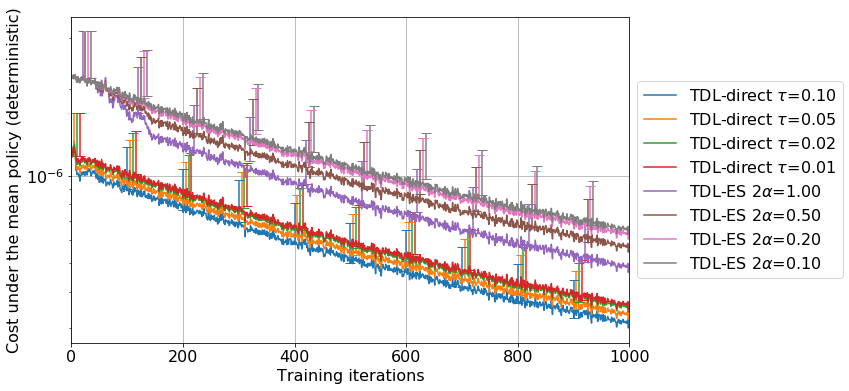}    
   \caption{Median performance of TDL algorithms out of 100 independent runs in the toy environment described in Section \ref{sec:instability}. The error bars indicate the worse 10\% quantiles.}
   \label{fig:toy_sensitivity}
\end{figure*}

In the toy domain described in Section \ref{sec:instability}, we run TDL-direct and TDL-ES with different values for the key hyperparameters. We find that our algorithms are robust to different hyperparameters and the training curves are all stable and smooth. We show the result in Figure \ref{fig:toy_sensitivity}.

We also analyze the hyperparameter sensitivity of TDL algorithms on Mujoco tasks. We show the results in Figure \ref{fig:sensitivity}. In the figure, we show the performances of the algorithms (average score over the first 4M steps) vs. the key hyperparameters on different Mujoco tasks. Notice that the average score is equivalent to the area under the learning curve. We observe that TDL algorithms are robust to different hyperparameters over a wide range.

\begin{figure*}[htbp]
   \centering
   \includegraphics[width=0.9\textwidth]{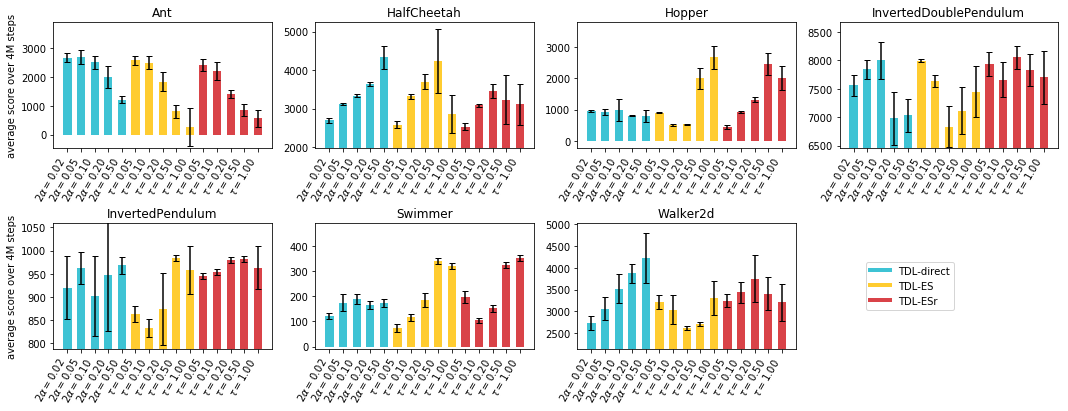}    
   \caption{The average score over the first 4M steps of TDL-direct, TDL-ES and TDL-ESr with different values for the key hyperparameters. The values for the hyperparameters are chosen such that the adjacent values approximately differ by a factor of two. The error bars indicate the standard deviation over five runs with different random seeds.}
   \label{fig:sensitivity}
\end{figure*}

\section{Instability Caused by Gradient Explosion}

In this section, we investigate whether PPO suffers from overly large gradients on Mujoco tasks and whether TDL performs better due to the reason that it avoids exploding gradient. To answer the questions, we record the L2 norm of the gradient during the training for PPO and TDL algorithms on \texttt{InvertedDoublePendulem}. We show the result in Figure \ref{fig:InvertedDoublePendulem}. We observe that the average gradient for PPO is significantly larger than those of TDL algorithms. This shows that TDL algorithms can avoid exploding gradient. Moreover, in PPO, the gradient increases after 500 iterations where the cumulative reward starts to oscillate. In contrast, TDL algorithms are stable and perform better. This indicates that the avoidance of exploding gradient in TDL algorithms is an important reason why it performs better.

\begin{figure*}[htbp]
   \centering
   \includegraphics[width=0.6\textwidth]{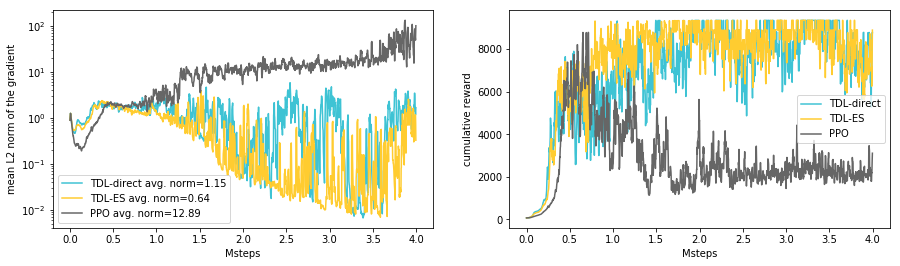}    
   \caption{Left. The gradient norm of the updates in each iteration during the training. The numbers in the legend indicate the average gradient norm over the training process. Right. The performance during the training.}
   \label{fig:InvertedDoublePendulem}
\end{figure*}

\section{Hyperparameters and Implementation Details}
\label{app:hyperparameters}

The hyperparameters in all the experiments for our algorithms accord to what are listed in Table \ref{tab:hyperparameters} unless otherwise stated. All the algorithms run on the \texttt{v2} tasks provided by OpenAI Gym. The source code can be downloaded from \url{https://github.com/targetdistributionlearning/target-distribution-learning/}.

For the experiments in Figure \ref{fig:convex}, the configuration of PPO is the same as the original paper \cite{schulman2017proximal} except that a SGD optimizer and simpler network structures are used. The policy network and the critic network are networks with one 10-neuron hidden layer. The configurations of TDL-direct and TDL-ES are the same as PPO in terms of the network structures, the number of samples in each iteration, the learning rate, etc. We did grid search over the hyperparameters for the remedies. In the remedy with minimum variance, a large minimum variance alleviates oscillation but increases the cost, and we chose the smallest value that can result in stable behaviors. In the remedy with a smaller clip constant, all the clip constants in a wide range result in oscillation, and we chose a value that delays the oscillation most. In the remedy with an entropy term, we chose the entropy coefficient that results in the lowest average cost.

For the experiments in Figure \ref{fig:benchmark}, the hyperparameters of our algorithms are listed in Table \ref{tab:hyperparameters}. The hyperparameters of the previous algorithms are based on the corresponding original papers and the key hyperparameters for these algorithms are tuned individually for each task. The key hyperparameters includes the initial standard deviation $\sigma_0$,
the constraint constant $\delta$ in TRPO, the clipping constant $\epsilon$ in PPO and the KL divergence constraint $\epsilon$ in MPO. We performed a grid search over these hyperparameters and selected the best for each task.


For the experiment in Figure \ref{fig:epochs}, the initial standard deviation of the action distribution is set to $1.0$ which is the same as that used in the compared PPO and \texttt{mu2\_max}($=2\alpha$) is set to $1.0$ to highlight the effect of TDL methods instead of TDL-direct. Moreover, to be fair in comparison with PPO where the variance of the action distribution is state independent, we use $\varphi=0$ in this set of experiments.


\begin{table*}[htbp]
  \caption{Hyperparameters in the experiments}
  \label{tab:hyperparameters}
  \centering
  \begin{tabular}{l|l}
    \toprule
    Hyperparameter & Value \\ 
    \midrule
    Policy network & 3 hidden layers with 64 neurons each \\
    Critic network & 3 hidden layers with 64 neurons each \\
    Num. steps in each iteration & 2048 \\
    Discount rate ($\gamma$) & 0.995 \\
    GAE parameter ($\lambda$) & 0.97 \\
    Num. epochs & 60 \\
    Minibatch size & 256 \\
    Adam learning rate & $1\times 10^{-4}$ \\
    Initial standard deviation of the action distribution ($\sigma_0$)& 0.3 \\
    \texttt{mu2\_max} ($=2\alpha$) in TDL-direct & $0.05$ \\ 
    Step size $\tau$ in TDL-ES and TDL-ESr & 1.0 (0.05 for \texttt{Ant} and \texttt{Humanoid}, \\
    & 0.5 for \texttt{HumanoidStandup}) \\
    Num. adjacent point $N$ in TDL-ESr & 2 (5 for \texttt{HumanoidStandup}) \\
    Revising ratio $r$ in TDL-ESr & 0.1 (1.0 for \texttt{HumanoidStandup})\\
    $\varphi$ for the variance update & 1.0 \\
    \bottomrule
  \end{tabular}
\end{table*}

\end{alphasection}

\end{document}